\documentclass[twoside,11pt]{article}

\usepackage{wrapfig}
\usepackage{multirow}
\usepackage{multicol}
\usepackage{booktabs}
\usepackage{colortbl}
\usepackage{subcaption}
\usepackage{cellspace}
\usepackage{lastpage}
\usepackage{jmlr2e}

\usepackage{amsmath,amsfonts,bm}

\def\eqref#1{equation~\ref{#1}}

\def\1{\bm{1}}

\def\vz{{\bm{z}}}

\def\mA{{\bm{A}}}

\def\mH{{\bm{H}}}

\DeclareMathAlphabet{\mathsfit}{\encodingdefault}{\sfdefault}{m}{sl}
\SetMathAlphabet{\mathsfit}{bold}{\encodingdefault}{\sfdefault}{bx}{n}
\newcommand{\tens}[1]{\bm{\mathsfit{#1}}}
\def\tA{{\tens{A}}}

\def\tL{{\tens{L}}}

\def\tX{{\tens{X}}}

\def\gG{{\mathcal{G}}}
\def\gH{{\mathcal{H}}}

\def\gS{{\mathcal{S}}}

\def\sR{{\mathbb{R}}}

\newcommand{\R}{\mathbb{R}}

\newcommand{\pan}[1]{{\color{blue}P: #1}}

\jmlrheading{26}{2025}{1-\pageref{LastPage}}{5/23; Revised 1/25}{1/25}{23-0560}{Xiyuan Wang, Pan Li, and Muhan Zhang}

\ShortHeadings{Improving Graph Neural Networks on Multi-node Tasks with the Labeling Trick}{Wang, Li, and Zhang}
\firstpageno{1}

\makeatletter
\newcommand\blfootnote[1]{%
  \begingroup
  \renewcommand{\@makefntext}[1]{\noindent\makebox[1.8em][r]#1}
  \renewcommand\thefootnote{}\footnote{#1}%
  \addtocounter{footnote}{-1}%
  \endgroup
}
\makeatother

\begin{document}

\title{Improving Graph Neural Networks on Multi-node Tasks with\\ the Labeling Trick}
\author{\name Xiyuan Wang \email wangxiyuan@pku.edu.cn \\
       \addr Institute for Artificial Intelligence\\
       Peking University\\
       Beijing, China
       \AND
       \name Pan Li \email panli@gatech.edu \\
       \addr School of Electrical and Computer Engineering\\
       Georgia Institute of Technology\\
       Atlanta, USA
       \AND
       \name Muhan Zhang$^*$
 \email muhan@pku.edu.cn \\
       \addr Institute for Artificial Intelligence\\
       Peking University\\
       Beijing, China
       }
\editor{Samy Bengio}
\maketitle

\begin{abstract}%
In this paper, we study using graph neural networks (GNNs) for \textit{multi-node representation learning}, where a representation for a set of more than one node (such as a link) is to be learned. Existing GNNs are mainly designed to learn single-node representations. When used for multi-node representation learning, a common practice is to directly aggregate the single-node representations obtained by a GNN. In this paper, we show a fundamental limitation of such an approach, namely the inability to capture the dependence among multiple nodes in the node set. A straightforward solution is to distinguish target nodes from others. Formalizing this idea, we propose \text{labeling trick}, which first labels nodes in the graph according to their relationships with the target node set before applying a GNN and then aggregates node representations obtained in the labeled graph for multi-node representations. Besides node sets in graphs, we also extend labeling tricks to posets, subsets and hypergraphs. Experiments verify that the labeling trick technique can boost GNNs on various tasks, including undirected link prediction, directed link prediction, hyperedge prediction, and subgraph prediction. Our work explains the superior performance of previous node-labeling-based methods and establishes a theoretical foundation for using GNNs for multi-node representation learning.
\end{abstract}


\section{Introduction}
\blfootnote{* correspondence to Muhan Zhang}Graph neural networks (GNNs)~\citep{scarselli2009graph,bruna2013spectral,duvenaud2015convolutional,li2015gated,kipf2016semi,defferrard2016convolutional,dai2016discriminative,velivckovic2017graph,zhang2018end,ying2018hierarchical} have achieved great successes in recent years. While GNNs have been well studied for single-node tasks (such as node classification) and whole-graph tasks (such as graph classification), using GNNs on tasks that involve multi-nodes is less studied and less understood. Among such \textit{multi-node representation learning} problems, link prediction (predicting the link existence/class/value between a set of two nodes) is perhaps the most important one due to its wide applications in practice, such as friend recommendation in social networks~\citep{adamic2003friends}, movie recommendation in Netflix~\citep{bennett2007netflix}, protein interaction prediction~\citep{qi2006evaluation}, drug response prediction~\citep{stanfield2017drug}, and knowledge graph completion~\citep{nickel2015review}. Besides link prediction, other multi-node tasks, like subgraph classification and hyperedge prediction, are relatively new but have found applications in gene set analysis~\citep{geneset}, user profiling~\citep{subgnn}, drug interaction prediction~\citep{familyset}, temporal network modeling~\citep{liu2022neural}, group recommendation~\cite{amer2009group}, etc. In this paper, we study the ability of GNNs to learn multi-node representations. As the link task is the simplest multi-node case, we mainly use link prediction in this paper to visualize and illustrate our method and theory. However, our theory and method apply generally to all multi-node representation learning problems such as subgraph~\citep{subgnn}, hyperedge~\citep{zhang2018beyond} and network motif~\citep{liu2022neural} prediction tasks.

Starting from the link prediction task, we illustrate the deficiency of existing GNN models for multi-node representation learning which motivates our labeling trick. There are two main classes of GNN-based link prediction methods: 
Graph AutoEncoder (GAE)~\citep{kipf2016variational} and SEAL~\citep{SEAL,li2020distance}.
\textbf{GAE} and its variational version VGAE~\citep{kipf2016variational} first apply a GNN to the entire graph to compute a representation for each node. The representations of the two end nodes of the link are then aggregated to predict the target link. On the contrary, SEAL assigns node labels according to their distances to the two end nodes before applying the GNN on the graph. SEAL often shows much better practical performance than GAE. The key lies in SEAL's node labeling step.

\begin{figure}[t]
\centering
\begin{subfigure}{0.4\textwidth}
\includegraphics[width=\textwidth]{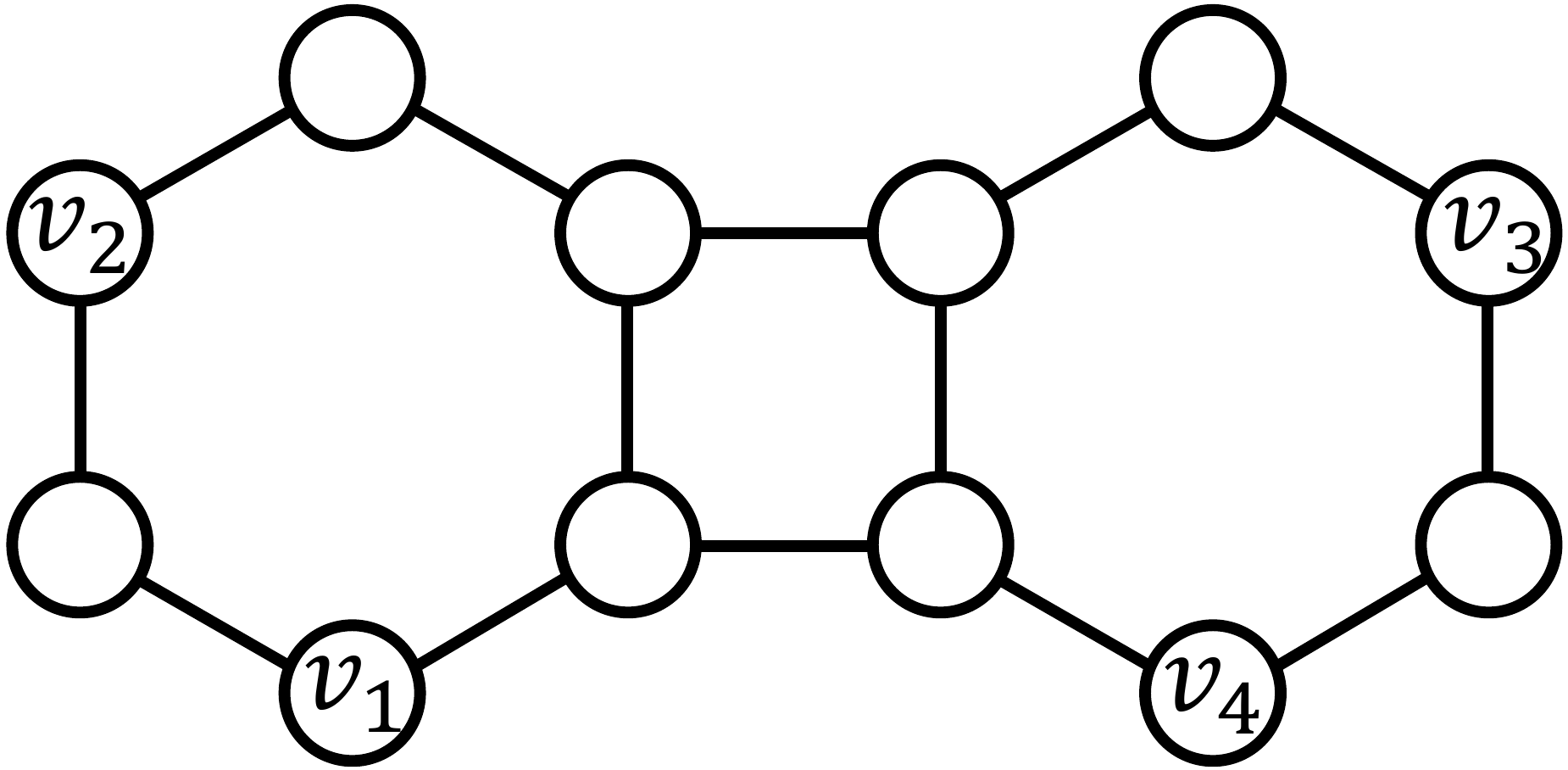}
\caption{}
\label{node_iso}
\end{subfigure}
\hspace{5pt}
\begin{subfigure}{0.22\textwidth}
\includegraphics[width=\textwidth]{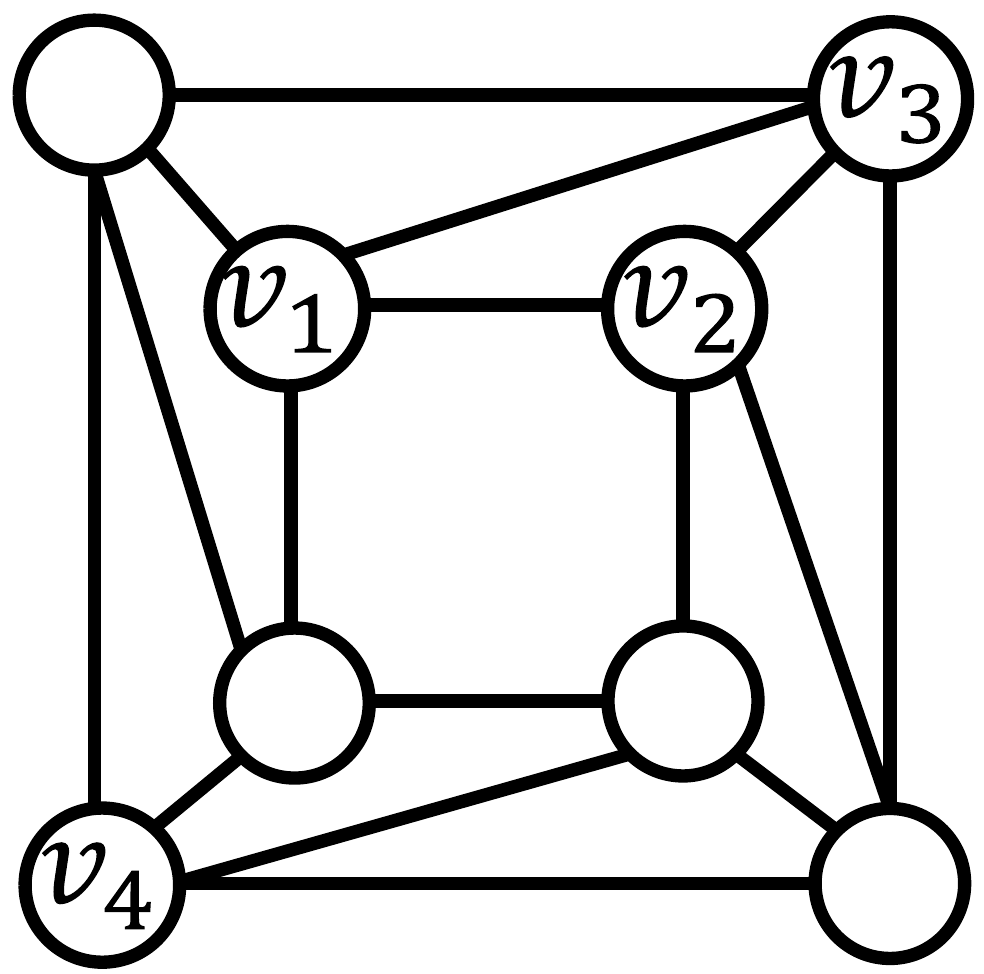}
\caption{}
\label{fig:subgexample}
\end{subfigure}
\caption{
(a) In this graph, nodes $v_2$ and $v_3$ are in the same orbit; links $(v_1,v_2)$ and $(v_4,v_3)$ are isomorphic; link $(v_1,v_2)$ and link $(v_1,v_3)$ are \textbf{not} isomorphic. However, if we aggregate two node representations learned by a GNN as the link representation, we will give $(v_1,v_2)$ and $(v_1,v_3)$ the same prediction. (b) In this graph, nodes $v_3$ and $v_4$ are isomorphic. Aggregating the node embeddings within the subgraph, GNNs will produce equal embeddings for subgraphs $(v_1,v_2,v_3)$ and $(v_1,v_2,v_4)$, while the two subgraphs are not isomorphic. This problem was first observed by~\cite{you2019position}, which was interpret as the failure of GNNs to capture node positions, and later became more formalized in~\citep{Srinivasan2020On}.}
\vspace{-15pt}
\end{figure}
We first give a simple example to show when GAE fails. In Figure~\ref{node_iso}, $v_2$ and $v_3$ have symmetric positions in the graph---from their respective views, they have the same $h$-hop neighborhood for any $h$. Thus, without node features, GAE will learn the same representation for $v_2$ and $v_3$. Therefore, when predicting which one of $v_2$ and $v_3$ is more likely to form a link with $v_1$, GAE will aggregate the representations of $v_1$ and $v_2$ as the link representation of $(v_1,v_2)$, and aggregate the representations of $v_1$ and $v_3$ to represent $(v_1,v_3)$, thus giving $(v_1,v_2)$ and $(v_1,v_3)$ the same representation and prediction. The failure to distinguish links $(v_1,v_2)$ and $(v_1,v_3)$ that have different structural roles in the graph reflects one key limitation of GAE-type methods: by computing $v_1$ and $v_2$'s representations independently of each other, GAE cannot capture the dependence between two end nodes of a link. For example, $(v_1,v_2)$ has a much smaller shortest path distance than that of $(v_1,v_3)$; and $(v_1,v_2)$ has both nodes in the same hexagon, while $(v_1,v_3)$ does not. We can also consider this case from another perspective. Common neighbor (CN)~\citep{liben2007link}, one elementary heuristic feature for link prediction, counts the number of common neighbors between two nodes to measure their likelihood of forming a link. It is the foundation of many other successful heuristics such as Adamic-Adar~\citep{adamic2003friends} and Resource Allocation~\citep{zhou2009predicting}, which are also based on neighborhood overlap. However, GAE cannot capture such neighborhood-overlap-based features. As shown in Figure~\ref{node_iso}, there is $1$ common neighbor between $(v_1,v_2)$ and $0$ between $(v_1,v_3)$, but GAE always gives $(v_1,v_2)$ and $(v_1,v_3)$ the same representation. The failure to learn common neighbor demonstrates GAE's severe limitation for link prediction. The root cause still lies in that GAE computes node representations independently of each other, and when computing the representation of one end node, it is unaware of the other end node.

In fact, GAE represents a common practice of using GNNs to learn multi-node representations. That is, obtaining individual node representations through a GNN and then aggregating the representations of those target nodes as the multi-node representation. Similar failures caused by independence of node representation learning also happen in general multi-node representation learning problems. In the subgraph representation learning task, which is to learn representations for subgraphs inside a large graph~\citep{subgnn}, representations aggregated from independently computed node representations will fail to differentiate nodes inside and outside the subgraph. Figure~\ref{fig:subgexample} (from ~\citet{GLASS}) shows an example. Directly aggregating node embeddings produced by a GNN will lead to the same representation for subgraphs $(v_1,v_2,v_3)$ and $(v_1,v_2,v_4)$. However, the former subgraph forms a triangle while the latter one does not.

This paper solves the above type of failures from a \textit{structural representation learning} point of view. We adopt and generalize the notion \textit{most expressive structural representation}~\citep{Srinivasan2020On}, which gives multi-node substructure the same representation if and only if they are \textit{isomorphic} (a.k.a. symmetric, on the same orbit) in the graph. For example, link $(v_1, v_2)$ and link $(v_4, v_3)$ in Figure~\ref{node_iso} are isomorphic, and a most expressive structural representation should give them the same representation. On the other hand, a most expressive structural representation will discriminate all non-isomorphic links (such as $(v_1, v_2)$ and $(v_1, v_3)$). According to our discussion above, GAE-type methods that directly aggregate node representations cannot learn a most expressive structural representation. Then, how to learn a most expressive structural representation of node sets?

To answer this question, we revisit the other GNN-based link prediction framework, SEAL, and analyze how node labeling helps a GNN learn better node set representations. We find that two properties of the node labeling are crucial for its effectiveness: 1) target-node distinguishing, which ensures that target nodes receive labels that differentiate them from other nodes in the graph and 2) permutation equivariance. With these two properties, we define \textit{set labeling trick}, which considers each multi-node substructure as a node set and unifies previous node labeling methods into a single and most general form. Theoretically, we prove that with set labeling trick, a sufficiently expressive GNN can learn most expressive structural representations of node sets (Theorem~\ref{thm:labeltrick}), which reassures GNN’s node set prediction ability. It also closes the gap between the nature of GNNs to learn node representations and the need of multi-node representation learning in node-set-based inference tasks. 

Set labeling trick is for multi-node structure of a node set and can be used on a wide range of tasks including link prediction and subgraph classification. However, to describe and unify even more tasks and methods, we propose three extensions of set labeling trick. One is \textit{poset labeling trick}. In some tasks, target nodes may have intrinsic order relations in real-world problems. For example, in citation graphs, each link is from the citing article to the cited one. In such cases, describing multi-node substructures with node sets leads to loss of order information. This motivates us to add order information to the label and use poset instead to describe substructures. Another extension is \textit{subset labeling trick}. It unifies labeling methods besides SEAL~\citep{SEAL}, like ID-GNN~\citep{you2021identity} and NBFNet~\citep{NBFNet}. These works label only a subset of nodes each time. We formalize these methods and analyze the expressivity: when using GNNs without strong expressivity, subset labeling trick exhibits higher expressivity than labeling tricks in some cases. Last but not least, by converting hypergraph to bipartite graph, we straightforwardly extend labeling trick to hypergraph.



\section{Preliminaries}
In this section, we introduce some important concepts that will be used in the analysis of the paper, including \textit{permutation}, \textit{poset-graph isomorphism} and \textit{most expressive structural representation}.

We consider a graph $\gG=(V,E,\tA)$, where $V=\{1,2,\ldots,n\}$ is the set of $n$ vertices, $E \subseteq V\times V$ is the set of edges, and $\tA \in \R^{n\times n \times k}$ is a 3-dimensional tensor containing node and edge features. In this paper, we let all graphs have a node set numbered from 1 to the total number of nodes in the graph. The diagonal components $\tA_{i,i,:}$ denote features of node $i$, and the off-diagonal components $\tA_{i,j,:}$ denote features of edge $(i,j)$. The node/edge types can also be expressed in $\tA$ using integers or one-hot encoding vectors for heterogeneous graphs. We further use $\mA \in \{0,1\}^{n\times n}$ to denote the adjacency matrix of $\gG$ with $\mA_{i,j}=1$ iff $(i,j)\in E$, where it is possible $\mA_{i,j}\neq \mA_{j,i}$. We let $\mA$ be the first slice of $\tA$, i.e., $\mA = \tA_{:,:,1}$. Since $\tA$ contains the complete information of a graph, we also directly denote the graph by $\tA$.

\subsection{Permutation}
The same graph can index nodes in different orders, and these different indices can be connected with permutation.
\begin{definition}\label{def:permutation} 
A \textbf{permutation} $\pi$ is a bijective mapping from $\{1,2,\ldots,n\}$ to $\{1,2,\ldots,n\}$. All $n!$ possible $\pi$'s constitute the permutation group $\Pi_n$. 
\end{definition}

Depending on the context, permutation $\pi$ can mean assigning a new index $\pi(i)$ to node $i\in V$, or mapping node $i$ to node $\pi(i)$ of another graph. Slightly extending the notation, we let the permutation of a set/sequence denote permuting each element in the set/sequence. For example, permutation $\pi$ maps a set of nodes $S\subseteq V$ to $\pi(S)=\{\pi(i)|i\in S\}$ and maps a set of node pairs $S'\subseteq V\times V$ to $\pi(S')=\{\pi((i,j))|(i,j)\in S'\}=\{(\pi(i),\pi(j))|(i,j)\in S'\}$. The permutation of a graph's tensor $\tA$, denoted as $\pi(\tA)$, can also be defined: $\pi(\tA)_{\pi(i),\pi(j)}=\tA_{i, j}$, where original $i$-th node and $j$-th node will have new index $\pi(i),\pi(j)$ while keeping the features of the pair $\tA_{i, j}$. 

Permutation is closely related to \textit{graph isomorphism}, whether two graphs describe the same structure. Intuitively, as nodes in graphs have no order, no matter what permutation is applied to a graph, the transformed graph should be isomorphic to the original graph. Similarly, if one graph can be transformed into another under some permutation, the two graphs should also be isomorphic. Formally speaking, 
\begin{definition}
Two graphs $\tA\in \sR^{n\times n\times d}, \tA'\in \sR^{n'\times n'\times d'}$ are \textbf{isomorphic} iff there exists $\pi\in \Pi_n$, $\pi(\tA)=\tA'$. 
\end{definition}
In whole graph classification tasks, models should give isomorphic graphs the same prediction as they describe the same structure, and differentiate non-isomorphic graphs.

\subsection{Poset-Graph Isomorphism}

To describe a substructure defined by a subset of nodes with internal relation, like a directed edge, we introduce poset. A poset is a set with a partial order. Partial order is a reflexive, antisymmetric, and transitive homogeneous relation on the set~\citep{Posetdef}.
\begin{definition}\label{def:poset}
A \textbf{poset} $S$ is a tuple $(U, \le_S)$, where $U$ is a set, and $\le_S\subseteq U\times U$ is a relation on $U$. Let $u\le_S v$ denote $(u, v)\in \le_S$. $\le_S$ fulfills the following conditions.
\begin{enumerate}
    \item Reflexivity. $\forall v\in U, v\le_S v$.
    \item Antisymmetry. $\forall u, v\in U$, if $u\le_S v$ and $v\le_S u$, then $u=v$.
    \item Transitivity. $\forall u, v, w \in U$, if $u\le_S v$ and $v\le_S w$, then $u\le_S w$.
\end{enumerate}
\end{definition}
The permutation operation on partial order relation and poset is defined as follows.
\begin{align*}
    \pi(\le_S) = \pi(\{(u, v) ~|~ (u, v) \in \le_S\}) &= \{(\pi(u), \pi(v)) ~|~ (u, v) \in \le_S\},\\
    \pi(S) = \pi((U, \le_S)) &= (\pi(U), \pi(\le_S)).
\end{align*}

To describe when two posets derive the same substructure, we define \textit{poset-graph isomorphism}, which generalizes graph isomorphism to arbitrary node posets in a graph.
\begin{definition}\textbf{(Poset-graph isomorphism)}\label{def:poset_iso}
    Given two graphs $\gG=(V,E,\tA)$, $\gG'=(V',E',\tA')$, and two node posets $S=(U, \le_S),U\subseteq V$, $S'=(U',\le_{S'}),U'\subseteq V'$, we say substructures $(S,\tA)$ and $(S', \tA')$ are isomorphic (denoted by $(S,\tA) \simeq (S',\tA')$) iff ~$\exists \pi \in \Pi_n, S = \pi(S')$ and $\tA = \pi(\tA')$.
\end{definition}

A set is a particular case of poset, where the partial order only contains reflexive relations $u\le_S u,u\in U$. It can describe substructures without order, like undirected edges and subgraphs. Abusing the notation of poset, we sometimes also use $S$ to denote a set and omit the trivial partial order relation. Then, \textit{set-graph isomorphism} is defined as follows.

\begin{definition}\label{def:set_iso}
\textbf{(Set-graph isomorphism)} Given two graphs $\gG=(V,E,\tA)$, $\gG'=(V',E',\tA')$, and two node sets $S\subseteq V$, $S'\subseteq V'$, we say substructures $(S,\tA)$ and $(S',\tA')$ are isomorphic (denoted by $(S,\tA) \simeq (S',\tA')$) iff ~$\exists \pi \in \Pi_n, S = \pi(S')$ and $\tA = \pi(\tA')$.
\end{definition}

Note that both set- and poset-graph isomorphism are \textbf{more strict} than graph isomorphism. They not only need a permutation which maps one graph to the other but also require the permutation to map a specific node poset $S$ to $S'$. 

In practice, when the target node poset does not contain all nodes in the graph, we are often more concerned with the case of $\tA=\tA'$, where isomorphic node posets are defined \textbf{in the same graph}. For example, when $S = \{i\}, S' = \{j\}$ and $(i,\tA) \simeq (j,\tA)$, we say nodes $i$ and $j$ are isomorphic in graph $\tA$ (or they have symmetric positions/same structural role in graph $\tA$). An example is $v_2$ and $v_3$ in Figure~\ref{node_iso}. Similarly, edge and subgraph isomorphism can also be defined as the isomorphism of their node posets.

\subsection{Structural Representations}
Graph models should produce the same prediction for isomorphic substructures. We define permutation invariance and equivariance to formalize this property. A function $f$ defined over the space of $(S,\tA)$ is \textit{permutation invariant} (or \textit{invariant} for abbreviation) if $\forall \pi \in \Pi_n$, $f(S,\tA) = f(\pi(S),\pi(\tA))$. Similarly, $f$ is \textit{permutation equivariant} if $\forall \pi \in \Pi_n$, $\pi(f(S,\tA)) = f(\pi(S),\pi(\tA))$, where for example $f(S, \tA)$ can be a tensor $L \in \sR^{n\times n\times d}$, $\pi(f(S, \tA))_{\pi(i)\pi(j)}=f(S, \tA)_{ij}$. Permutation invariance/equivariance ensures that representations learned by a GNN are invariant to node indexing, a fundamental design principle of GNNs. 

Now we define the \textit{most expressive structural representation} of a substructure $(S,\tA)$, following \citep{Srinivasan2020On,li2020distance}. It assigns a unique representation to each equivalence class of isomorphic substructures.

\begin{definition}\label{def:structuralrepresentation}
Given an invariant function $\Gamma(\cdot)$ mapping node subsets in graphs to a latent space, $\Gamma(\cdot)$ is a \textbf{most expressive structural representation}, if ~$\forall S,\tA, S',\tA'$, $ \Gamma(S,\tA) = \Gamma(S',\tA') \Leftrightarrow (S,\tA) \simeq (S',\tA')$.
\end{definition}

For simplicity, we will directly use \textit{structural representation} to denote most expressive structural representation in the rest of the paper. We will omit $\tA$ if it is clear from context. For a graph $\tA$, we call $\Gamma(\tA)=\Gamma(\emptyset, \tA)$ a \textit{structural graph representation}, $\Gamma(i,\tA)$ a \textit{structural node representation} for node $i$, and call $\Gamma(\{i,j\},\tA)$ a \textit{structural link representation} for link $(i,j)$. For a general node poset $S$, we call $\Gamma(S,\tA)$ a \textit{structural multi-node representation} for $S$.

Definition~\ref{def:structuralrepresentation} requires that the structural representations of two substructures are the same if and only if the two substructures are isomorphic. That is, isomorphic substructures always have the \textbf{same} structural representation, while non-isomorphic substructures always have \textbf{different} structural representations. Due to the permutation invariance requirement, models should not distinguish isomorphic substructures. This implies that structural representations can discriminate all substructures that any invariant model can differentiate, and structural representations reach the highest expressivity. 



\section{The Limitation of Directly Aggregating Node Representations}
In this section, taking GAE for link prediction as an example, we show the critical limitation of directly aggregating node representations as a multi-node representation.

\setlength{\abovedisplayskip}{5pt}
\setlength{\belowdisplayskip}{5pt}

\subsection{GAE for Multi-Node Representation}

GAE~\citep{kipf2016variational} is a kind of link prediction model with GNN. Given a graph $\tA$, GAE first uses a GNN to compute a node representation $\vz_i$ for each node $i$, and then use the inner product of $\vz_i$ and $\vz_j$ to predict link $\{i,j\}$:
\begin{align*}
    \hat{\mA}_{i,j} = \text{sigmoid}(\vz_i^\top\vz_j), ~\text{where}~\vz_i \!=\! \text{GCN}(i,\tA), \vz_j \!=\! \text{GCN}(j,\tA).\nonumber
\end{align*}
Here $\hat{\mA}_{i,j}$ is the predicted score for link $\{i,j\}$. The model is trained to maximize the likelihood of reconstructing the true adjacency matrix. The original GAE uses a two-layer GCN~\citep{kipf2016semi}. In principle, we can replace GCN with any GNN, use any aggregation function over the set of target node embeddings including mean, sum, and max other than inner product, and substitute $\text{sigmoid}$ with an MLP. Then, GAE can be used for multi-node tasks. It aggregates target node embeddings produced by the GNN:
\begin{equation*}
    \vz_{S}=\text{MLP}(\text{AGG}(\{\vz_i\mid i\in S\})) ~\text{where}~\vz_i \!=\! \text{GNN}(i,\tA),
\end{equation*}
where $\text{AGG}$ is an aggregation function, which takes a multiset instead of set by default. We will use GAE to denote this general class of GNN-based multi-node representation learning methods in the following. Two natural questions are: 1) Is the node representation learned by the GNN a \textit{structural node representation}? 2) Is the multi-node representation aggregated from a set of node representations a \textit{structural representation for the node set}? We answer them respectively in the following.

\subsection{GNN and Structural Node Representation}
Practical GNNs~\citep{gilmer2017neural} usually simulate the 1-dimensional Weisfeiler-Lehman (1-WL) test~\citep{weisfeiler1968reduction} to iteratively update each node's representation by aggregating its neighbors' representations. We use \textit{1-WL-GNN} to denote a GNN with 1-WL discriminating power, such as GIN~\citep{xu2018powerful}.

A 1-WL-GNN ensures that isomorphic nodes always have the same representation. However, the opposite direction is not guaranteed. For example, a 1-WL-GNN gives the same representation to all nodes in an $r$-regular graph, in which non-isomorphic nodes exist. Despite this, 1-WL is known to discriminate almost all non-isomorphic nodes as the number of nodes grows to infinity~\citep{babai1979canonical}, which indicates that a 1-WL-GNN can give different representations to almost all non-isomorphic nodes in large real-world graphs. 

To study GNN's maximum expressivity, we define a \textit{node-most-expressive (NME) GNN}, which gives different representations to \textbf{all} non-isomorphic nodes.
\begin{definition}
A \text{GNN} is \textbf{node-most-expressive (NME)} if there exists a parameterization of the GNN that ~$\forall i,\tA$,$j,\tA'$, $~\text{GNN}(i,\tA) = \text{GNN}(j,\tA') \Leftrightarrow (i,\tA) \simeq (j,\tA')$.
\end{definition}
NME GNN learns \textit{structural node representations}. We define such a GNN because our primary focus is on multi-node representation. By ignoring the limitations of single-node expressivity, NME GNN simplifies our analysis. Although a polynomial-time implementation is not known for NME GNNs, many practical software tools can discriminate between all non-isomorphic nodes efficiently~\citep{mckay2014practical}, providing a promising direction. 

\subsection{GAE \textbf{Cannot} Learn Structural Multi-Node Representations}\label{gae_limitation}
Suppose GAE is equipped with an NME GNN producing structural node representations. 
Then the question becomes: does the aggregation of structural node representations of the target nodes result in a structural representation of the target node set? The answer is no. We have already illustrated this problem in the introduction: In Figure~\ref{node_iso}, we have two isomorphic nodes $v_2$ and $v_3$, and thus $v_2$ and $v_3$ will have the same structural node representation. By aggregating structural node representations, GAE will give $(v_1,v_2)$ and $(v_1,v_3)$ the same link representation. However, $(v_1,v_2)$ and $(v_1,v_3)$ are not isomorphic in the graph. Figure~\ref{fig:subgexample} gives another example on the multi-node case involving more than two nodes. Previous works have similar examples~\citep{Srinivasan2020On,Zhang2020Inductive}. All these results indicate that:

\begin{proposition}
(\citet{Srinivasan2020On}) GAE \textbf{cannot} learn structural multi-node representations no matter how expressive node representations a GNN can learn.
\end{proposition}

The root cause of this problem is that GNN computes node representations independently without being aware of the other nodes in the target node set $S$. Thus, even though GNN learns the most expressive single-node representations, there is never a guarantee that their aggregation is a structural representation of a node set. In other words, the multi-node representation learning problem is \textbf{not breakable} into multiple \textbf{independent} single-node representation learning problems. We need to consider the \textbf{dependency} between the target nodes when computing their single-node representations. %

\section{Labeling Trick for Set}

Starting from a common case in real-world applications, we first describe the multi-node substructure defined by a node set (instead of a poset) in the graph and define \textit{set labeling trick}. The majority of this part is included in our conference paper~\citep{zhang2021labeling}.  
\subsection{Definition of Set Labeling Trick}\label{sec:labelingtrick}
The set labeling trick is defined as follows.
\begin{definition}\label{def:setlabelingtrick}
\textbf{(Set labeling trick)} For a graph $\tA$ and a set $S$ of nodes in the graph, we stack a labeling tensor $\tL(S,\tA) \in \R^{n\times n\times k}$ in the third dimension of $\tA$ to get a new $\tA^{(S)} \in \R^{n\times n\times (k+d)}$. $\tL$ satisfies: $\forall S,\tA, S',\tA',\pi \in \Pi_n$,
\begin{enumerate}
    \item (\textit{target-nodes-distinguishing}) ~$\tL(S,\tA) = \pi(\tL(S', \tA')) \Rightarrow S = \pi(S')$.
    \item (\textit{permutation equivariance}) ~~$S = \pi(S'), \tA = \pi(\tA') \Rightarrow \tL(S,\tA) = \pi(\tL(S',\tA'))$.
\end{enumerate}
\end{definition}

To explain a bit, labeling trick assigns a label vector to each node/edge in graph $\tA$, which constitutes the labeling tensor $\tL(S,\tA)$. By concatenating $\tA$ and $\tL(S,\tA)$, we get the new labeled graph $\tA^{(S)}$. By definition, we can assign labels to both nodes and edges. However, in this paper, we \textbf{consider node labels only} by default for simplicity, i.e., we let the off-diagonal components $\tL(S,\tA)_{i,j,:},i\neq j$ be all zero.

The labeling tensor $\tL(S,\tA)$ should satisfy two properties in Definition~\ref{def:setlabelingtrick}. Property 1 requires that if a permutation $\pi$ preserving node labels (i.e., $\tL(S,\tA) = \pi(\tL(S',\tA'))$) exists between nodes of $\tA$ and $\tA'$, then the nodes in $S'$ must be mapped to nodes in $S$ by $\pi$ (i.e., $S = \pi(S')$). A sufficient condition for property 1 is to make the target nodes $S$ have \textit{distinct labels} from those of the rest nodes so that $S$ is \textit{distinguishable} from others.
Property 2 requires that when $(S,\tA)$ and $(S',\tA')$ are isomorphic under $\pi$ (i.e., $S = \pi(S'), \tA = \pi(\tA')$), the corresponding nodes $i\in S, j\in S', i=\pi(j)$ must always have the same label (i.e., $\tL(S,\tA) = \pi(\tL(S',\tA'))$). A sufficient condition for property 2 is to make the labeling function \textit{permutation equivariant}, i.e., when the target $(S,\tA)$ changes to $(\pi(S),\pi(\tA))$, the labeling tensor $\tL(\pi(S),\pi(\tA))$ should equivariantly change to $\pi(\tL(S,\tA))$.

\subsection{How Labeling Trick Works}
Obviously, labeling trick puts extra information into the graph, while the details remain unclear. To show some intuition on how labeling trick boosts graph neural networks, we introduce a simplest labeling trick satisfying the two properties in Definition~\ref{def:setlabelingtrick}.

\begin{definition}\label{def:zolabeling}
\textbf{(Zero-one labeling trick)} Given a graph $\tA$ and a set of nodes $S$ to predict, we give it a diagonal labeling matrix $\tL_{zo}(S,\tA) \in \R^{n\times n\times 1}$ such that 
\begin{equation*}
    \tL_{zo}(S,\tA)_{i,i,1}=\begin{cases}
        1&\text{if } i\in S\\
        0&\text{otherwise}
    \end{cases}.
\end{equation*}
\end{definition}

In other words, the zero-one labeling trick assigns $1$ to nodes in $S$ and labels $0$ to all other nodes in the graph. It is a valid labeling trick because nodes in $S$ get \textit{distinct labels} from others, and the labeling function is \textit{permutation equivariant} by always giving nodes in the target node set label $1$. These node labels serve as additional node features fed to a GNN together with the original node features.

Let's return to the example in Figure~\ref{node_iso} to see how the zero-one labeling trick helps GNNs learn better multi-node representations. This time, when we want to predict link $(v_1,v_2)$, we will label $v_1,v_2$ differently from the rest nodes, as shown by the distinct colors in Figure~\ref{link_iso}~left. When computing $v_2$'s representation, GNN is also ``aware'' of the source node $v_1$ with nodes $v_1$ and $v_2$ labeled, rather than treating $v_1$ the same as other nodes. Similarly, when predicting link $(v_1,v_3)$, the model will again label $v_1,v_3$ differently from other nodes as shown in Figure~\ref{link_iso}~right. This way, $v_2$ and $v_3$'s node representations are no longer the same in the two differently labeled graphs (due to the presence of the labeled $v_1$), and the model can predict $(v_1,v_2)$ and $(v_1,v_3)$ differently. The key difference of model with labeling trick from GAE is that the node representations are no longer computed independently, but are \textit{conditioned} on each other in order to capture the dependence between nodes.
\begin{figure}[tp]
\centering
\includegraphics[width=0.8\textwidth]{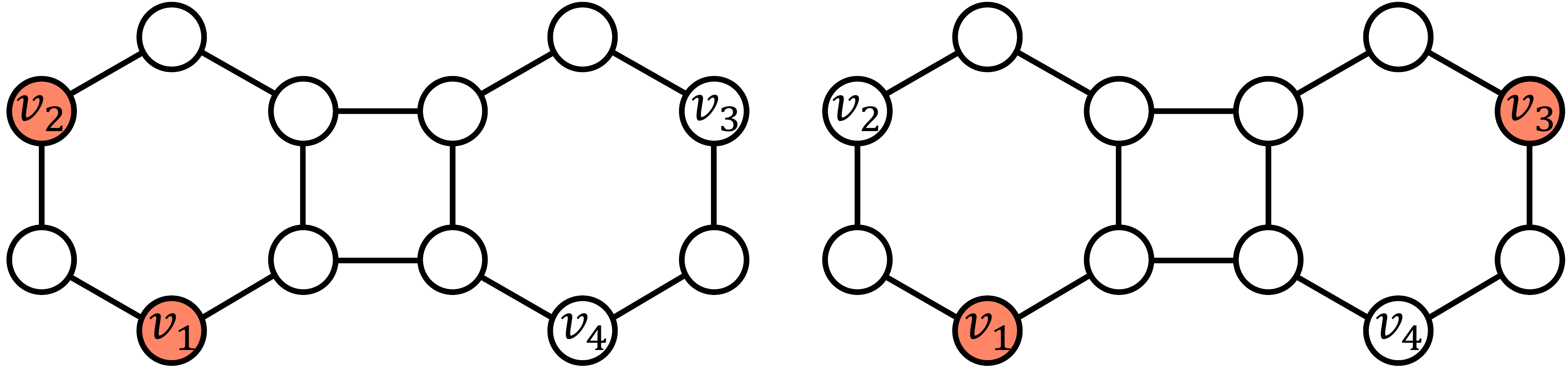}
\caption{\small When predicting $(v_1,v_2)$, we will label these two nodes differently from the rest so that a GNN is aware of the target link when learning $v_1$ and $v_2$'s representations. Similarly, when predicting $(v_1,v_3)$, nodes $v_1$ and $v_3$ will be labeled differently. This way, the representation of $v_2$ in the left graph will be different from that of $v_3$ in the right graph, enabling GNNs to distinguish the non-isomorphic links $(v_1,v_2)$ and $(v_1,v_3)$.}
\label{link_iso}
\end{figure}

\subsection{Expressivity of GNN with Labeling Trick}

We include all proofs in the appendix. 

Labeling trick first bridges the gap between whole-graph representation (the focus of graph level GNNs) and node set representations.

\begin{proposition}\label{prop::gsiso2giso}
   \citep{zhang2021labeling} For any node set $S$ in graph $\tA$ and $S'$ in graph $\tA'$, given a set labeling trick, $(S, \tA)\simeq (S', \tA')\Leftrightarrow \tA^{(S)}\simeq {\tA'}^{(S')}$. 
\end{proposition}
The problem of graph-level tasks on a labeled graph ($\tA^{(S)}$ as defined in Definition~\ref{def:setlabelingtrick}) is equivalent to that of multi-node tasks. However, the complexity of these graph-level GNNs are usually larger than GNNs encoding nodes. We further want to connect node set representations with node representations. Now we introduce our main theorem showing that with a valid labeling trick, an NME GNN can \textit{learn structural representations of node sets}.
\begin{theorem}\label{thm:labeltrick}
\citep{zhang2021labeling} Given an \text{NME} GNN and an injective set aggregation function \text{AGG}, for any $S,\tA, S',\tA'$, $\text{GNN}(S,\tA^{(S)}) = \text{GNN}(S',\tA'^{(S')}) \Leftrightarrow (S,\tA) \!\simeq\! (S',\tA') \nonumber$, where $\text{GNN}(S,\tA^{(S)}) := \text{ AGG}(\{\text{GNN}(i,\tA^{(S)}) \mid i\in S\})$.
\end{theorem}

Remember that directly aggregating the structural node representations learned from the original graph $\tA$ does not lead to structural representations of node sets (Section~\ref{gae_limitation}). In contrast, Theorem~\ref{thm:labeltrick} shows that aggregating the structural node representations learned from the \textbf{labeled} graph $\tA^{(S)}$, somewhat surprisingly, results in a structural representation for $(S,\tA)$.

The significance of Theorem~\ref{thm:labeltrick} is that it closes the gap between the nature of GNNs for single-node representations and the requirement of multi-node representations for node set prediction problems. Although GNNs alone have severe limitations for multi-node representations, GNNs + labeling trick can learn structural representations of node sets by aggregating structural node representations obtained in the labeled graph. 

Theorem~\ref{thm:labeltrick} assumes an NME GNN. To augment Theorem~\ref{thm:labeltrick}, we give the following theorems, which demonstrate the power of labeling trick for 1-WL-GNNs on link prediction.

\begin{theorem} \label{thm:link-aso-num}
\citep{zhang2021labeling} Given an $h$-layer 1-WL-GNN, in any non-attributed graph with $n$ nodes, if the degree of each node in the graph is between $1$ and $\big((1-\epsilon)\log n\big)^{1/(2h+2)}$ for any constant $\epsilon\in \left(\frac{\log \log n}{(2h+2)\log n}, 1\right)$, there exists $\omega(n^{2\epsilon})$ pairs of non-isomorphic links $(u,w), (v,w)$ such that 1-WL-GNN gives $u,v$ the same representation, while with 1-WL-GNN + zero-one labeling trick gives $u,v$ different representations.
\end{theorem}

Theorem~\ref{thm:link-aso-num} shows that in any non-attributed graph there exists a large number ($\omega(n^{2\epsilon})$) of link pairs (like the examples $(v_1,v_2)$ and $(v_1,v_3)$ in Figure~\ref{node_iso}) which are not distinguishable by 1-WL-GNNs alone but distinguishable by 1-WL-GNNs + labeling trick. This means, labeling trick can boost the expressive power of 1-WL-GNNs on link prediction tasks. 

How labeling trick boosts link prediction can also be shown from another perspective: 1-WL-GNN + zero-one labeling trick can \textbf{learn various link prediction heuristics} while vanilla 1-WL-GNN cannot. 

\begin{proposition}\label{prop:cnraaa}
\sloppy Given a link prediction heuristic of the following form, 
\begin{align*}
    f\big(\big\{\sum_{v\in N(i)}g_2(\text{deg}(v,\tA)),
    \sum_{v\in N(j)}g_2(\text{deg}(v,\tA))\big\},\!\!
    \sum_{v\in N(i)\bigcap N(j)}\!\!g_1(\text{deg}(v,\tA))\big),
\end{align*}
where $deg(v,\tA)$ is the degree of node $v$ in graph $\tA$, $g_1,g_2$ are positive functions, and $f$ is injective w.r.t. the second input with the first input fixed. There exists a 1-WL-GNN + zero-one labeling trick implementing this heuristic. In contrast, 1-WL-GNN cannot implement it.
\end{proposition} 
The heuristic defined in the above proposition covers many widely-used and time-tested link prediction heuristics, such as common neighbors (CN)~\citep{CommonNeighbor}, resource allocation(RA)~\citep{zhou2009predicting}, and Adamic-Adar(AA)~\citep{adamic2003friends}. These important structural features for link prediction are not learnable by vanilla GNNs but can be learned if we augment 1-WL-GNNs with a simple zero-one labeling trick.

Labeling trick can also boost graph neural networks in subgraph tasks with more than two nodes. The following proposition.
\begin{proposition}
\label{thm:BoostWLGNNsubg}(\citet{GLASS})
Given an $h$-layer 1-WL-GNN, in any non-attributed graph with $n$ nodes, if the degree of each node in the graph is between $1$ and $\big((1-\epsilon)\log n\big)^{1/(2h+2)}$ for any constant $\epsilon>0$, there exists $w(2^n n^{2\epsilon-1})$ pairs of non-isomorphic subgraphs such that that 1-WL-GNN produces the same representation, while 1-WL-GNN + labeling trick can distinguish them.
\end{proposition}
Theorem~\ref{thm:BoostWLGNNsubg} extends Theorem~\ref{thm:link-aso-num} to more than 2 nodes. It shows that an even larger number of node set pairs need labeling tricks to help 1-WL-GNNs differentiate them. 

\subsection{Complexity}
Despite the expressive power, labeling trick may introduce extra computational complexity. The reason is that for every node set $S$ to predict, we need to relabel the graph $\tA$ according to $S$ and compute a new set of node representations within the labeled graph. In contrast, GAE-type methods compute node representations only in the original graph. 

Let $m$ denote the number of edges, $n$ denote the number of nodes, and $q$ denote the number of target node sets to predict. As node labels are usually produced by some fast non-parametric method, we neglect the overhead for computing node labels. Then we compare the inference complexity of GAE and GNN with labeling trick. For small graphs, GAE-type methods can compute all node representations first and then predict multiple node sets at the same time, which saves a significant amount of time. In this case, GAE's time complexity is $O(m+n+q)$, while GNN with labeling trick takes up to $O(q(m+n))$ time. However, for large graphs that cannot fit into the GPU memory, extracting a neighborhood subgraph for each node set to predict has to be used for both GAE-type methods and labeling trick, resulting in similar computation cost $O(q(n_s+m_s))$, where $n_s, m_s$ are the average number of nodes and edges in the segregated subgraphs. We also measures time and GPU memory consumption on link prediction task in Appendix~\ref{app::time}.

\begin{wrapfigure}[16]{L}{0.45\textwidth}
\centering
\vskip -14pt
\includegraphics[width=0.2\textwidth]{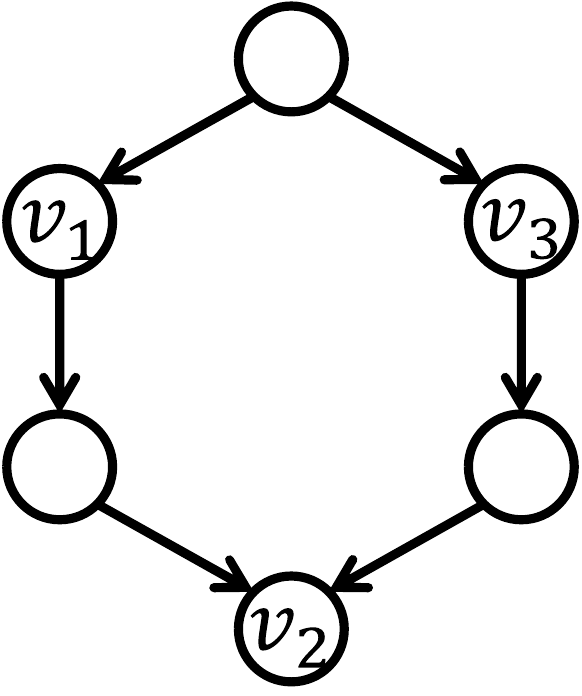}
\vskip -9pt
\caption{\small Set labeling with Graph Neural Networks (GNNs) fails to distinguish between non-isomorphic directed links, such as the edge from $v_1$ to $v_2$ versus the edge from $v_2$ to $v_3$, because it does not account for the order of nodes within the target node pairs. }
\label{fig:directedlink}
\end{wrapfigure}

\section{Labeling Trick for Poset}
The previous section describes multi-node substructures $(S,\tA)$ defined by node set $S$, which assumes that nodes in $S$ have no order relation. However, the assumption may lose some critical information in real-world tasks.  
For example, the citing and cited articles should be differentiated in citation graphs.
As shown in Figure~\ref{fig:directedlink}, using set labeling trick cannot discriminate the link direction by giving the two directed links the same representation, yet the two directed links are obviously non-isomorphic. Therefore, introducing order relation into node set is necessary for substructures with internal relation. In this section, we use poset to define multi-node substructures and extend set labeling trick to \textit{poset labeling trick}. Note that node order is only additionally introduced for $S$ because the graph $\tA$ already allows directed edges in our definition. 
\begin{definition}\label{def:Lposet}
\textbf{(Poset labeling trick)} Given a graph $\tA$ and a poset $S$ of nodes in it, we stack a labeling tensor $\tL(S,\tA) \in \R^{n\times n\times d}$ in the third dimension of $\tA$ to get a new $\tA^{(S)} \in \R^{n\times n\times (k+d)}$, where $\tL$ satisfies: for all poset $S$ of nodes in graph $\tA$, poset $S'$ of nodes in graph $\tA'$, and $\pi \in \Pi_n$,
\begin{enumerate}
    \item (\textit{target-nodes-and-order-distinguishing}) ~$\tL(S,\tA) = \pi(\tL(S',\tA)) \Rightarrow S = \pi(S')$.
    \item (\textit{permutation equivariance}) ~~$S = \pi(S'), \tA = \pi(\tA') \Rightarrow \tL(S,\tA) = \pi(\tL(S', \tA'))$. 
\end{enumerate}
\end{definition}
The definition of poset labeling trick is nearly the same as that of set labeling trick, except that we require permutation of poset and poset-graph isomorphism (Definition~\ref{def:poset} and~\ref{def:poset_iso}). Poset labeling trick still assigns a label vector to each node/edge in graph $\tA$. The labels distinguish the substructure from other parts of the graph and keep permutation equivariance. As we will show, poset labeling trick enables maximum expressivity for poset learning.  Below we first discuss how to design poset labeling tricks that satisfy the two above properties.

\subsection{Poset Labeling Trick Design} 

To describe general partial order relations between nodes in a poset, we introduce \textit{Hasse diagram}, a graph that uniquely determines the partial order relation.

\begin{definition}\label{def:hassediag} 
The Hasse diagram of a poset $S=(U,\le_S)$, denoted as $\gH_{S}$, is a directed graph $(V_H, E_H)$, $V_H=U$, $E_H=\{(u, v)\mid v\neq u \text{ and }v\text{ covers }u\}$, where $v$ covers $u$ means that $u\le_S v$ and there exists no $w\in U,w\notin \{u,v\}$, $u\le_S w$ and $w\le_S v$. 
\end{definition}
\begin{wrapfigure}[12]{L}{0.5\textwidth}
\vskip -0.2in
\centering
\includegraphics[width=0.48\textwidth]{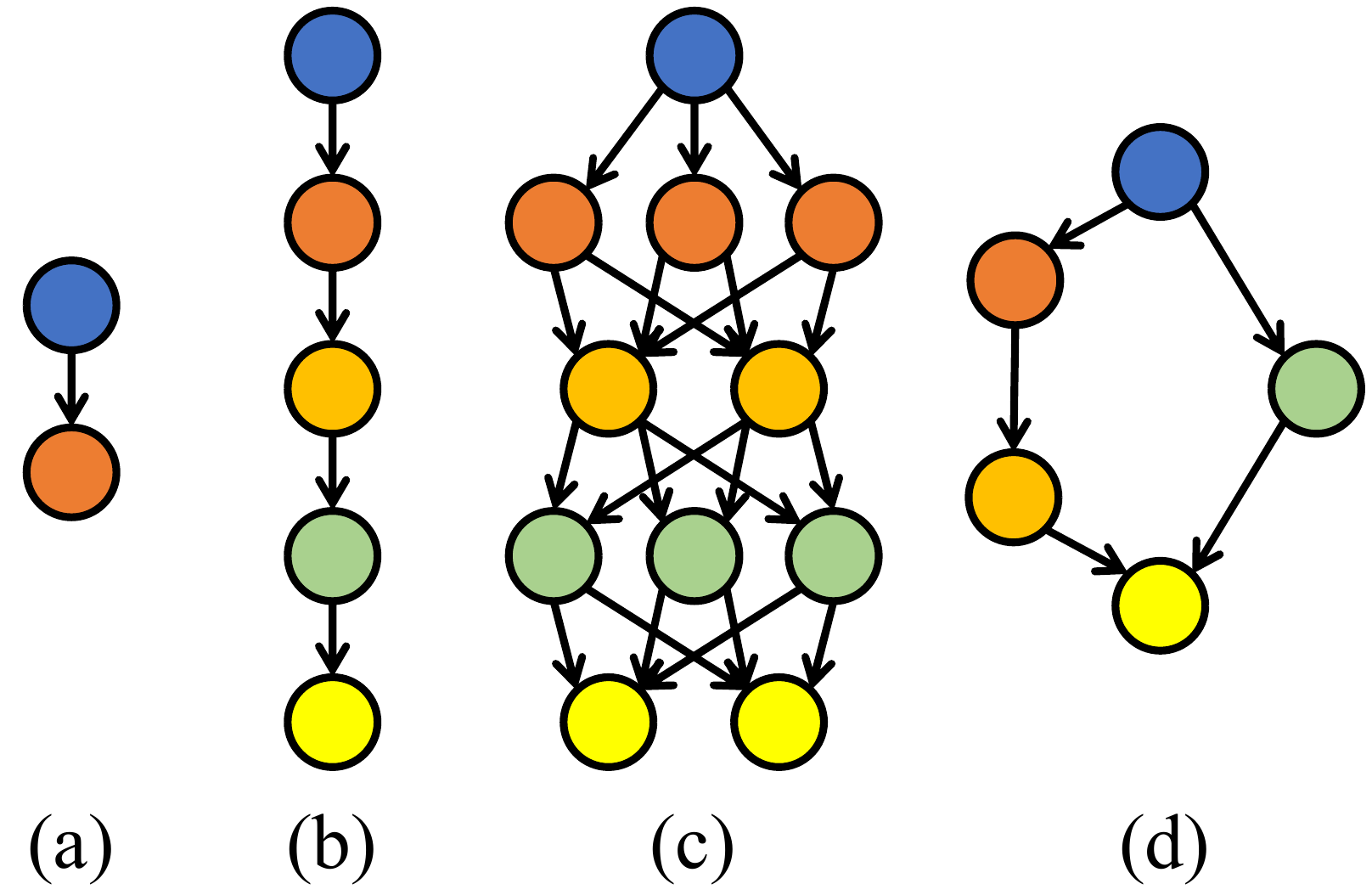}
\caption{\small Different Hasse diagrams}\label{fig:hassediag}
\end{wrapfigure}
Figure~\ref{fig:hassediag} shows some examples of Hasse diagram.
The reason we use Hasse diagram to encode partial order relation is that we prove any poset labeling trick satisfying Definition~\ref{def:Lposet} must give non-isomorphic nodes in a Hasse diagram different labels.
\begin{proposition}\label{prop:pohasse}
Let $\tL$ be the labeling function of a poset labeling trick. If $\exists \pi\in\Pi_n, \tL(S,\tA)=\pi(\tL(S',\tA'))$, then for all $v'\in S'$, $\pi(v')$ is in $S$, and $(\{v'\}, \gH_{S'})\simeq (\{\pi(v')\}, \gH_{S})$. Furthermore, in the same $\gH_{S}$, non-isomorphic nodes must have different labels.
\end{proposition}

Proposition~\ref{prop:pohasse} shows that a valid poset labeling trick should differentiate non-isomorphic nodes in a Hasse diagram. 
Theoretically, we can run an NME GNN on the Hasse diagram so that the node embeddings can serve the purpose.
Such a poset labeling trick is defined as follows. 
\begin{definition}\label{def:polabelingtrick}
Given an NME GNN, Hasse embedding labeling trick is
\begin{equation*}
\tL(S, \tA)_{u,u,:}=
    \begin{cases}
    \text{iso}(u, \gH_{S})&\text{if } u\in S\\
    0& \text{otherwise}
    \end{cases}
\end{equation*}
\end{definition} 
This labeling trick fulfills the two requirements in Definition~\ref{def:Lposet}. $\textit{iso}(u, \gH_{S})$ denotes the isomorphism type (non-zero) of node $u$ in Hasse diagram $\gH_{S}$, where $\textit{iso}(u_1, \gH_{S_1})=\textit{iso}(u_2, \gH_{S_2})$ iff $(u_1, \gH_{S_1})\simeq(u_2, \gH_{S_2})$. Hasse embedding labeling trick is similar to the zero-one labeling trick for set in Definition~\ref{def:zolabeling}. It assigns nodes outside the target poset the same label and distinguishes nodes inside based on their isomorphism class in the Hasse diagram, while the zero-one labeling trick does not differentiate nodes inside the poset. 

The above poset labeling trick can work on posets with arbitrary complex partial orders, at the cost of first identifying node isomorphism types in the Hasse diagram. In most real-world tasks, differentiating non-isomorphic nodes in Hasse diagrams is usually quite easy. For example, in the directed link prediction task, the target posets all have same simple Hasse diagram: only two roles exist in the poset---source node and target node of the link, which is shown in Figure~\ref{fig:hassediag}(a). Then we can assign a unique color to each equivalent class of isomorphic nodes in the Hasse diagram as the node labels, e.g., giving 1 to the source node, 2 to the target node, and 0 to all other nodes in directed link prediction. We can also design other simple poset labeling tricks. Two cases are discussed in the following.

\noindent\textbf{Linear Order Set.} Linear order set means a poset whose each pair of nodes are comparable, so that the Hasse diagram is a chain as shown in Figure~\ref{fig:hassediag}(b). Therefore, $S$ can be sorted in $u_1\le_S u_2\le_S u_3\le_S...\le_S u_k$, where $S = (U, \le_S), U=\{u_1,u_2,...,u_k\}$. Then we can assign $u_i$ label $i$ and give nodes outside $S$ $0$ label. Such a labeling trick is a valid poset labeling trick and can be used to learn paths with different lengths.

\noindent\textbf{Nearly Linear Order Set.} Nearly linear order set means there exists a partition of $S$, $\{S_1,S_2,...,S_l\}$, $\le_S=\bigcup_{i=1}^{l-1}S_i\times S_{i+1}$. As shown in Figure~\ref{fig:hassediag}(c), the Hasse diagram is nearly a chain whose nodes are replaced with a set of nodes with no relations. We can assign nodes in $S_i$ label $i$ and give nodes outside $S$ $0$ label. It is still a valid poset labeling trick. Nearly linear order set can describe a group in an institute, where the top is the leader.

\subsection{Poset Labeling Trick Expressivity}
We first show that poset labeling trick enables maximum expressivity for poset learning.

\begin{proposition}\label{prop::psgiso2giso}
    For any node poset $S$ in graph $\tA$ and $S'$ in graph $\tA'$, given a set labeling trick, $(S, \tA)\simeq (S', \tA')\Leftrightarrow \tA^{(S)}\simeq {\tA'}^{(S')}$. 
\end{proposition}
Proposition~\ref{prop::psgiso2giso} shows that structural poset representation is equivalent to the structural whole graph representation of labeled graph. Poset labeling trick can also bridge the gap between node representations and poset representations.

\begin{theorem}\label{thm:poexpressivity}
Given an \text{NME GNN} and an injective aggregation function $\text{AGG}$, for any node posets $S, S'$ in graphs $\tA,\tA'$, $\text{GNN}(S,\tA^{(S)}) = \text{GNN}(S',\tA'^{(S')}) \Leftrightarrow (S,\tA) \!\simeq\! (S',\tA') \nonumber$, where $\text{GNN}(S,\tA^{(S)})=\text{AGG}(\{\text{GNN}(u,\tA^{(S)}\mid u\in S\}))$.
\end{theorem}
Theorem~\ref{thm:poexpressivity} shows that with an NME GNN, poset labeling trick will produce structural representations of posets. To augment this theorem, we also discuss 1-WL-GNNs with poset labeling trick. 1-WL-GNNs cannot capture any partial order information and cannot differentiate arbitrary different posets with the same set of nodes. Differentiating different posets with different sets is also hard for 1-WL-GNNs as they fail to capture relations between nodes. Poset labeling trick can help in both cases. 
\begin{proposition}\label{thm:BoostWLGNNonPoset}
In any non-attributed graph with $n$ nodes, if the degree of each node in the graph is between $1$ and $\big((1-\epsilon)\log n\big)^{1/(2h+2)}$ for any constant $\epsilon>0$, there exist $w(n^{2\epsilon})$ pairs of links and $w((n!)^2)$ pairs of non-isomorphic node posets such that any $h$-layer 1-WL-GNN produces the same representation, while with Hasse embedding labeling trick 1-WL-GNN can distinguish them.
\end{proposition}
Proposition~\ref{thm:BoostWLGNNonPoset} illustrates that poset labeling trick can help 1-WL-GNNs distinguish significantly more pairs of node posets.

\section{Subset Labeling Trick for Multi-Node Representation Learning}
Besides set labeling trick, there exist other methods that append extra features to the adjacency to boost GNNs. Among them, ID-GNN~\citep{you2021identity} and NBFNet~\citep{NBFNet} assign special features to only one node in the target node set and also achieve outstanding performance. In this section, we propose subset labeling trick. As its name implies, subset labeling trick assigns labels only to a subset of nodes in the target node set. We compare set labeling trick with subset labeling trick in different problem settings. In some cases, subset labeling trick is even more expressive than set labeling trick.
\subsection{Subset Labeling Trick}
Similar to set labeling trick, subset labeling trick also have two properties. 
\begin{definition}\label{def:subsetlabelingtrick}
\textbf{(subset labeling trick)} Given set $S$ in graph $\tA$ and its subset $P\subseteq S$, we stack a labeling tensor $\tL(P, \tA) \in \R^{n\times n\times d}$ in the third dimension of $\tA$ to get a new $\tA^{(P)} \in \R^{n\times n\times (k+d)}$, where $\tL$ satisfies: $\forall S,\tA, S',\tA', P\subseteq S, P'\subseteq S',\pi \in \Pi_n$,
\begin{enumerate}
    \item (\textit{target-subset-distinguishing}) ~$\tL{(P, \tA)} = \pi(\tL{(P',\tA')}) \Rightarrow P = \pi(P')$.
    \item (\textit{permutation equivariance}) ~~$P=\pi(P'), \tA = \pi(\tA') \Rightarrow \tL{(P, \tA)} = \pi(\tL{(P', \tA')})$. 
\end{enumerate}
\end{definition}

Like set labeling trick, subset labeling trick distinguishes the selected subset in the target set and keeps permutation equivariance. However, it does not need to distinguish all target nodes. Subset($k$) labeling trick means the subset size is $k$.

Subset zero-one labeling trick is a simplest subset labeling trick fulfilling the requirements in Definition~\ref{def:subsetlabelingtrick}.

\begin{definition}\label{def:subsetzolabeling}
\textbf{(Subset zero-one labeling trick)} Given a graph $\tA$, a set of nodes $S$ to predict, and a subset $P\subseteq S$, we give it a diagonal labeling matrix $\tL(P, \tA) \in \R^{n\times n\times 1}$ such that $\tL(P, \tA)_{i,i,1} = 1$ if $i\in P$ and $\tL(P, \tA)_{i,i,1} = 0$ otherwise.
\end{definition}

To explain a bit, the subset zero-one labeling trick assigns label $1$ to nodes in the selected subset $P$, and label $0$ to all nodes not in $P$. It only contains the subset identity information.

Then a natural problem arises: how to select subset $P$ from the target node set $S$? Motivated by previous methods, we propose two different routines: subset-pooling and one-head.

\subsection{How to Select Subset}
\subsubsection{Subset Pooling}

ID-GNN~\citep{you2021identity} proposes an a GNN for node set learning. For each node in the target node set, it labels the node one and all other nodes zero. Then, it uses a 1-WL-GNN to produce the representations of the node. By pooling all node representations, ID-GNN produces the node set representation. As isomorphic node sets can have different embeddings due to different subset selections, choosing only one node randomly can break permutation equivariance. But pooling the representation of all subset selection eliminates the non-determinism caused by selection and solves this problem.  Generalizing this method, we propose the \textit{subset pooling routine}. Subset($k$) pooling enumerates all size-$k$ subsets and then pools the embeddings of them.
\begin{equation*}
    \text{AGG} (\{\text{GNN}(S, \tA^{(P)})\mid P\subseteq S, |P|=k\}),  
\end{equation*}
where \text{AGG} is an injective set aggregation function.

As for all $\pi\in\Pi_n$ and target node set $S$ in graph $\tA$, 
\begin{equation*}
    \text{AGG}(\!\{\text{GNN}(S,\tA^{\!(P)\!})\mid P\!\subseteq \!S, |P|\!=\!k\}\!)=\text{AGG}(\!\{\text{GNN}(\pi(S), \pi(\tA)^{\!(P)\!})\mid P\!\subseteq\!\pi(S),|P|\!=\!k\}\!),
\end{equation*}
the subset pooling routine keeps permutation equivariance.

\subsubsection{One Head Routine}
Contrary to the subset pooling routine, link prediction model NBFNet~\citep{NBFNet} labels only one head of the link. This design breaks permutation equivariance but improves the scalability. 
We propose the \textit{one head routine} to generalize this method to general node set tasks. It selects only one subset to label. Some policies are shown in the following.

\begin{itemize}
    \item \textit{Random Selection.} For a target set, we can select a subset in it randomly. For example, we can randomly choose one head of each target edge in link prediction task. 
    \item \textit{Graph Structural Selection.} We can select a node with maximum degree in the target node set. Note that it cannot keep permutation equivariance either.
    \item \textit{Partial Order Relation Selection.} If the least element exists in a poset, we can choose it as the subset. For example, in directed link prediction task, the source node of each link can be the subset. This method can keep permutation equivariance. 
\end{itemize}

\subsubsection{Complexity} The efficiency gain of subset labeling trick compared with set labeling trick comes from sharing results across target node sets. 
GNN with set labeling trick has to compute the representations of each target node set separately. With the target node distinguishing property, no labeling trick can remain unchanged across different target nodes sets. Therefore, the input adjacency will change and node representations have to be reproduced by the GNN. 

In contrast, GNN with subset labeling trick can compute the representations of multiple node sets with the same selected subset simultaneously. The subset label is only a function of the selected subset and the graph, so we can maintain the subset label for different target node sets by choosing the same subset. For example, in link prediction task, all links originating from a node share this same source node. By choosing the source node as the subset, these links have the same label and input adjacency to GNN, so the node representations produced by the GNN can be reused. This routine is especially efficient in the knowledge graph completion setting, where a query involves predicting all possible tail entities connected from a head entity with a certain relation.

\subsection{Expressivity}
When the subset size $k$ equals the target node set size $|S|$, subset labeling trick is equivalent to set labeling trick. What is more interesting is, when $k=|S|-1$, subset labeling trick with the subset pooling routine can achieve the same power as set labeling trick.
\begin{theorem}\label{thm:k-1plabelingexpressivity}
Given an \text{NME GNN}, for any graph $\tA,\tA'$, and node sets $S,S'$ in $\tA,\tA'$ respectively, we have
\begin{multline}
\text{AGG}(\!\{\!\text{GNN}(S, \tA^{(\!P\!)})\mid P\!\subseteq\!S, |P|\!=\!|S|-1\!\}\!)\!=\!\text{AGG}(\!\{\!\text{GNN}(S', \tA'^{(\!P'\!)})\mid P'\!\subseteq\!S', |P'|\!=\!|S'|-1\!\}\!) \\\Leftrightarrow (S,\tA) \!\simeq\! (S',\tA').
\end{multline}
\end{theorem}
Theorem~\ref{thm:k-1plabelingexpressivity} illustrates that when the selected subset is of $|S|-1$ size, GNNs can produce structural representation with the subset-pooling routine. This theorem is especially useful when $|S|=2$, in other words, link prediction task. Labeling only one node each time and pooling the two results can achieve the same high expressivity.

Under the one head routine, we have the following theorem.
\begin{theorem}\label{thm:k-1plabelingexpressivity2}
Given an \text{NME GNN}, for any graph $\tA,\tA'$, and node sets $S,S'$ in $\tA,\tA'$ respectively, we have
\begin{multline}
(S,\tA) \!\not\simeq\! (S',\tA')\Rightarrow \\
\forall P\subseteq S, P'\subseteq S', |P|=|S|-1, |P'|=|S'|-1,
\text{GNN}(S, \tA^{(P)})\neq \text{GNN}(S, \tA'^{(P')}).    
\end{multline}
\end{theorem}
Though one-head routine may produce different representations for isomorphic sets, the above theorem shows that it maintains the capacity to differentiate non-isomorphic sets.  


For larger target node set, subset($|S|-1$) labeling trick is of little use, as the $|S|-1$ labeling can hardly be reused by other target sets. In contrast,  we focus on the expressivity of subset($1$) labeling trick, since it is much more common for target node sets to share node rather than sharing another $(|S|-1)$ node set.

When using NME GNN, according to Theorem~\ref{thm:labeltrick}, set labeling trick leads to the highest expressivity. The problem left is whether subset($1$) labeling trick can help NME GNN produce structural representations.
\begin{proposition}\label{prop:NME GNN-L>PL}
Given an \text{NME GNN}, 
there exists pairs of set $S$ in graph $\tA$ and set $S'$ in graph $\tA'$ such that $\text{AGG}(\{\text{GNN}(u,\tA^{(u)})\mid u\in S\})=\text{AGG}(\{\text{GNN}(u',\tA'^{(u')})\mid u'\in S'\})$ while $(S,\tA)\!\not\simeq\! (S',\tA')$.
\end{proposition} 
Proposition~\ref{prop:NME GNN-L>PL} shows that with NME GNN, subset($1$) labeling trick cannot learn structural representation and is less expressive than set labeling trick. However, using 1-WL-GNNs, the expressivity of subset($1$) labeling trick is incomparable to that of set labeling trick. In other words, there exists non-isomorphic node sets which are distinguishable by subset($1$) labeling trick and indistinguishable by set labeling trick, and vice versa.
\begin{proposition}\label{prop:WLGNN-L<>PL}
\sloppy Given a 1-WL-GNN, there exists $S,\tA, S',\tA'$ such that  $(S,\tA)\!\not\simeq\! (S',\tA') \nonumber$, $\text{AGG}(\{\text{GNN}(u,\tA^{(u)})\mid u\in S\})\neq\text{AGG}(\{\text{GNN}(u',\tA'^{(u')})\mid u'\in S'\})$ while $\text{GNN}(S, \tA^{S})=\text{GNN}(S', {\tA'}^{(S')})$. There also exists $S,\tA, S',\tA'$ such that  $(S,\tA)\!\not\simeq\! (S',\tA') \nonumber$, $\text{AGG}(\{\text{GNN}(u,\tA^{(u)})\mid u\in S\})=\text{AGG}(\{\text{GNN}(u',\tA'^{(u')})\mid u'\in S'\})$ while $\text{GNN}(S, \tA^{(S)})\neq\text{GNN}(S', {\tA'}^{(S')})$.
\end{proposition} 
And 1-WL-GNN with subset($1$) labeling trick can also differentiate many pairs of node sets that 1-WL-GNN cannot differentiate, as shown in the following theorem.
\begin{proposition}\label{thm:PlabelingBoostWLGNN}
In any non-attributed graph with $n$ nodes, if the degree of each node in the graph is between $1$ and $\big((1-\epsilon)\log n\big)^{1/(2h+2)}$ for any constant $\epsilon\in \left(\frac{\log \log n}{(2h+2)\log n}, 1\right)$, there exist $w(n^{2\epsilon})$ pairs of links and $w(2^n n^{3\epsilon-1})$ pairs of non-isomorphic node sets such that any $h$-layer 1-WL-GNN produces the same representation, while with subset(1) labeling trick 1-WL-GNN can distinguish them.
\end{proposition}

\subsubsection{Why Subset Labeling Trick Outperforms Labeling Trick in Some Cases?}\label{sec:set_vs_subset}
In this section, we take a closer look at some special cases and then give some intuitions on subset labeling trick and set labeling trick. NME GNN is too expressive to show some weakness of set labeling trick, so we focus on 1-WL-GNN. 


Subset labeling trick helps differentiate nodes with the same label. Taking the two graphs in Figure~\ref{fig:PWL} as an example, the target set is the whole graph. With zero-one labeling trick, 1-WL-GNN cannot differentiate them as all nodes in the two graphs have the same rooted subtree (see Figure 5a). However, subset zero-one labeling trick can solve this problem. The rooted subtree in the first graph always contains a nodes with label $1$, whereas in the second graph, the rooted subtree may sometimes contain no labeled nodes, leading to different 1-WL-GNN embeddings.
\begin{figure}[t]
\centering
\includegraphics[scale=0.45]{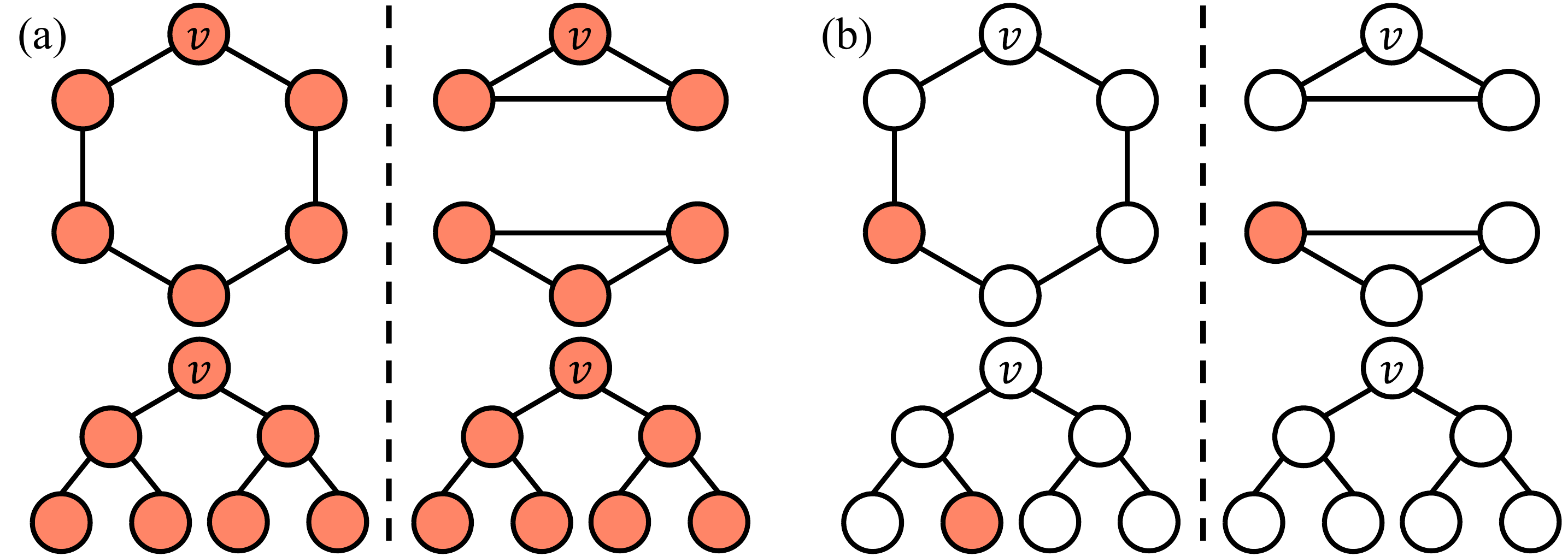}
\caption{An example of when subset labeling trick differentiates two node sets, while set labeling trick does not. First row: labeled graphs. Second row: rooted subtrees of $v$.}\label{fig:PWL}
\end{figure}
\begin{figure}[t]
\centering
\includegraphics[scale=0.45]{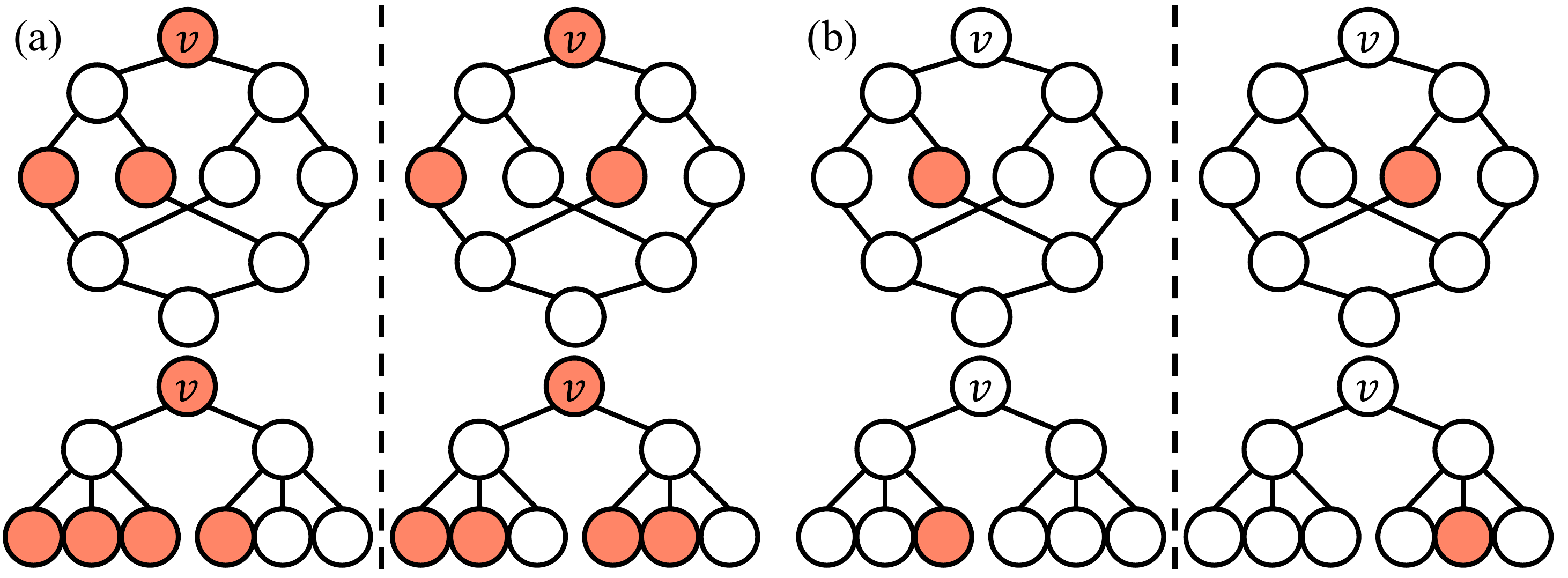}
\caption{An example of when subset labeling trick fails to differentiate two node sets while set labeling trick does. First row: labeled graphs. Second row: rooted subtrees of $v$.}\label{fig:PWL3}
\end{figure}

The drawback of subset labeling trick is that it captures pair-wise relation only and loses high-order relations. As shown in Figure~\ref{fig:PWL3}, the two target node sets (each containing three nodes) are non-isomorphic, but every node pair from the first set is isomorphic to a node pair from the second set. This difference is also reflected in the rooted subtree of target nodes (see the bottom of Figure~\ref{fig:PWL3}), where set labeling trick (Figure~\ref{fig:PWL3}a) can differentiate $v$ while subset($1$) labeling trick (Figure~\ref{fig:PWL3}b) cannot.
{
\section{Comparison between Labeling Trick and High-Order Graph Neural Network}\label{sec::labelingtrick_hognn}

Unlike ordinary GNNs which produce single-node representations, High-Order Graph Neural Networks (HOGNNs) generate representations for node tuples. HOGNNs encompass various approaches, including $k$-dimensional Graph Neural Networks ($k$-GNNs)~\citep{kWL} inspired by the $k$-dimensional Weisfeiler-Lehman test ($k$-WL)~\citep{cai1992optimal}, Provably Powerful Graph Neural Networks~\citep{PPGN} based on the $2$-dimensional folklore Weisfeiler-Lehman test (2-FWL), $k$-Invariant Graph Networks ($k$-IGN)~\citep{IGN}, Local Relational Pooling methods~\citep{LocalRelPool} that create permutation-invariant functions with adjacency matrices as input, and subgraph GNNs~\citep{ESAN,GNNAK,NGNN,OSAN,SSWL} which apply ordinary $1$-WL-GNNs to subgraphs extracted from the original graph.

These methods all target whole-graph tasks by pooling the generated node tuple representations to graph representations, whereas our labeling trick are designed for multi-node tasks. Nevertheless, HOGNNs also yield representations for node tuples and can be employed for multi-node tasks. Moreover, their ability to handle multi-node tasks is closely linked to their effectiveness in whole-graph tasks as follows.
\begin{proposition}
Let $c(S, \mathbf{A})$ denote the color produced by a HOGNN for a graph $G=(V, E, \mathbf{A})$ and a node tuple $S\in V^k$ in the graph. Let $\text{AGG}$ denote an injective pooling function. Given two graphs $G_1=(V_1, E_1, \mathbf{A}_1), G_2=(V_2, E_2, \mathbf{A}_2)$, $\text{AGG}(\{\!\{ c(S, \mathbf{A}_1) \mid  S\in V_1^k\}\!\})\neq \text{AGG}(\{\!\{ c(S, \mathbf{A}_2) \mid  S\in V_2^k\}\!\})$ (HOGNNs can differentiate the two graphs) is equivalent to the node tuple embedding function $c$ being able to differentiate two multisets of node tuples in two graphs.
\end{proposition}
A direct corollary is that if there exist two non-isomorphic graphs that a HOGNN cannot differentiate, then there exist two node tuples that the tuple representations output by HOGNN cannot differentiate. Moreover, if the node tuple embedding function $c_1$ is more expressive than $c_2$, such that $c_1(S_1, \mathbf{A}_1)=c_1(S_2, \mathbf{A}_2)\Rightarrow c_2(S_1, \mathbf{A}_1)=c_2(S_2, \mathbf{A}_2)$, the HOGNN corresponding to $c_1$ is also more expressive than that corresponding to $c_2$. Therefore, we can establish expressivity comparisons between labeling tricks for multi-node representations with HOGNNs by comparing their node tuple embedding functions. Following \citet{IDMPNN}, we first define the comparison between two HOGNNs for whole-graph representations.
\begin{definition}
For any algorithm $A$ and $B$, we denote the final color of graph $G$ computed by them as $c_A(G)$ and $c_B(G)$. We say:
\begin{itemize}
    \item $A$ is \textbf{more expressive} than $B$ ($B\preceq A$) if for any pair of graphs $G$ and $H$, $c_A(G)=c_A(H)\Rightarrow c_B(G)=c_B(H)$. Otherwise, there exists a pair of graphs that $B$ can differentiate while $A$ cannot, denoted as $B\not\preceq A$. 
    \item $A$ is \textbf{as expressive as} $B$ ($A\cong B$) if $B\preceq A \land A\preceq B$.
    \item $A$ is \textbf{strictly more expressive} than $B$ ($B\prec A$) if $B\preceq A \ \land \ A\ncong B$, i.e., for any pair of graphs $G$ and $H$, $c_A(G)=c_A(H)\Rightarrow c_B(G)=c_B(H)$, and there exists at least one pair of graphs $G, H$ s.t. $c_B(G)=c_B(H),c_A(G)\neq c_A(H)$.
    \item $A$ and $B$ are \textbf{incomparable} ($A\nsim B$) if $A \npreceq B \ \land \ B\npreceq A$. In this case, $A$ can distinguish a pair of non-isomorphic graphs that cannot be distinguished by $B$ and vice versa.
\end{itemize}
\end{definition}

$k$-dimensional Weisfeiler-Lehman (k-WL) test has strong expressivity and forms the basis of HOGNNs' expressivity hierarchy. It assigns colors to all $k$-tuples and iteratively updates them. The initial colors $c_k^{0}(S, G)$ of tuples $S\in V(G)^k$ are determined by their isomorphism types~\citep{PPGN}. Two tuples $S \in [n]^k$ in graph $G$, and $S' \in [n]^k$ in graph $G'$ receive the same isomorphism type if and only if (1) there exists a permutation function $\pi$ such that $\pi(S_i) = S'_i$ for all $i=1,2,...,k$; and (2) the subgraphs $G[S]$ and $G'[S']$ induced by tuples $S$ and $S'$ (with nodes $S_i$ in $G[S]$ and nodes $S'_i$ in $G'[S']$ assigned extra label $i$ correspondingly for $i=1,2,...,k$) are isomorphic. At the $t$-th iteration, the color updating scheme is
\begin{equation*}
c_k^{t}(S, G)=\text{Hash}(c_k^{t-1}(S, G), (\{\!\{
    c_k^{t-1}(\psi_i(S, u), G)\mid u\in V(G)\}\!\}\mid i\in [k])), 
\end{equation*}
where $\psi_i(S, u)$ means replacing the $i$-th element in $S$ with $u$. The color of $S$ is updated by its original color and the color of its high-order neighbors $\psi_i(S, u)$. The iterative update continues until the color converges, e.g. $\forall S, S'\in V(G)^k, c_k^{t+1}(S, G)=c_k^{t+1}(S', G)\Leftrightarrow c_k^{t}(S, G)=c_k^{t}(S', G)$. Let $c_k(S, G)$ denote the $k$-WL color of tuple $S$ in graph $G$. The color of the whole graph is the multiset of all tuple colors,
\begin{equation*}
    c_k(G)=\text{Hash}(\{\!\{{c_k(S, G)\mid S\in V(G)^k}\}\!\}).
\end{equation*}
$k$-WL can also be used to produce $l$-tuple representations ($l\le k$), Given $S\in V(G)^l$, its color is
\begin{equation*}
c_k(S)=\text{Hash}(\{\!\{{c_k(S\Vert S', G)\mid S'\in V(G)^{k-l}}\}\!\}),
\end{equation*}
where $\Vert$ means concatenation. The HOGNN corresponding to this tuple representation is $\text{Hash}(\{\!\{c_k(S)\mid S\in V(G)^l\}\!\}$. This algorithm, namely $k$-WL with $l$-pooling in this work, shares the same expressivity as $k$-WL.

\begin{proposition}\label{prop::kwl_lpool}
Given $l<k$, $k$-WL is as expressive as $k$-WL with $l$ pooling.
\end{proposition}
Therefore, we slightly abuse the notation of $k$-WL and use $k$-WL instead of $k$-WL with $l$ pooling when analyzing $k$-WL for $l$-tuple representation.

\subsubsection{$k, l$-WL and Poset Labeling Trick for Linear Order Set}

As most HOGNNs learn representations for node tuples, we first compare HOGNNs with the labeling trick for node tuples, where node tuple is essentially poset with linear order. We use $k,l$-WL~\citep{IDMPNN} to represent a general framework for HOGNNs, which includes $k$-WL-based methods~\citep{kWL}, subgraph GNNs~\citep{NGNN,GNNAK,OSAN,ESAN,SSWL} and relational pooling~\citep{LocalRelPool}.

\begin{definition}
    ($k,l$-WL) For a graph $G=(V, E, \mathbf A)$, the graph color produced by $k, l$-WL is as follows.
    \begin{enumerate}
        \item For each $l$-tuple of node $S$, the labeled graph $G^{S}$ is $G$ with the $i$-th node $S_i$ in tuple $S$ augmented with an additional feature $i$.
        \item Runs $k$-WL on each labeled graph $G^{S}$, leading to graph color $c_{k,l}(S, \mathbf A)$.
        \item The final color of graph $G$ is $\textnormal{HASH}(\{\!\{c_{k,l}(S, \mathbf A)\mid S\in V^l\}\!\})$.
    \end{enumerate}
\end{definition}

$k,l$-WL establishes a fine expressivity hierarchy for GNNs. 
\begin{proposition}\citep{IDMPNN}
    For all $k\ge 2, l\ge 0$
\begin{itemize}
    \item $k+1, l$-WL is strictly more expressive than $k, l+1$-WL.
    \item $k, l+1$-WL is strictly more expressive than $k, l$-WL.
    \item $k+1, l$-WL is strictly more expressive than $k, l$-WL.
    \item There exist two graphs that $2, l$-WL can differentiate while $l+1$-WL cannot.
\end{itemize}
\end{proposition}

Note that similar to $k,l$-WL, the poset labeling trick for linear order sets also assigns node indices in the tuple as additional node features and runs GNN on the augmented graph. Therefore, we have
\begin{corollary}\label{coro::klWL-poset}
Given two graphs $G_1=(V_1, E_1, \mathbf{A}_1), G_2=(V_2, E_2, \mathbf{A}_2)$ and two node tuples $S_1\in V_1^l, S_2\in V_2^l$, the poset labeling trick with $k$-WL equivalent GNN can differentiate $S_1,S_2$ if and only if $k,l$-WL produces different tuple colors $c_{k,l}(S_1,\mathbf{A}_1)\neq c_{k,l}(S_2,\mathbf{A}_2)$.
\end{corollary}
In other words, poset labeling trick for linear order sets (i.e., node tuples) combined with $k$-WL is equivalent in expressivity to $k,l$-WL where $l$ is the size of the set. Thus we can readily inherit the $k,l$-WL hierarchy to analyze labeling trick. For example, in real-world applications, the labeling trick is typically used with $1$-WL-GNNs, which have the same expressivity as $2$-WL. Therefore,
\begin{corollary}
Using 1-WL-GNN together with the poset labeling trick for linear order sets, for node tuples of size $l$, there exist two node tuples that $l+1$-WL cannot differentiate, while 1-WL-GNN with the labeling trick can differentiate. Moreover, for any two node tuples of size $l$, if $l+2$-WL cannot differentiate, 1-WL-GNN with the labeling trick cannot either.
\end{corollary}

Despite having the same expressivity, as the labeling trick only needs to compute representation of the query node tuple, it can be much more scalable than HOGNNs for multi-node tasks (saving $n^l$ times time and space, where $l$ is the size of node tuple).

\section{Labeling trick for hypergraph}
Graph is appropriate to describe bilateral relations between entities. However, high-order relations among several entities are also worth studying~\citep{HyperGIntro}. Hypergraph, composed of nodes and hyperedges, can model such high-order relations naturally. In this section, we study multi-node representation learning in hypergraphs.


We consider a hypergraph $H := (V,E,\mH, \tX^V, \tX^E)$, where $V$ is the node set $\{1,2,...,n\}$, $E$ is the hyperedge set $\{1,2,...,m\}$
, and $\mH \in \{0,1\}^{n\times m}$ is the incidence matrix with $\mH_{i,j}=1$ if node $i$ is in hyperedge $j$ and $0$ otherwise. Each hyperedge contains at least one node.
$\tX^{V}\in\sR^{n\times d}$ and $\tX^E\in \sR^{m\times d}$ are node and hyperedge features respectively, where $\tX^V_{i,:}$ is of node $i$, and $\tX^E_{j,:}$ is of hyperedge $j$. 

We define a hypergraph permutation $\pi=(\pi_1,\pi_2)\in \Pi_n\times \Pi_m$. Its action on a hypergraph $H=(V,E, \mH, \tX^V, \tX^E)$ is $\pi(H)=(\pi_1(V), \pi_2(E), \pi(\mH), \pi_1(\tX^V),\pi_2(\tX^E))$, where 
incidence matrix permutation is $\pi(\mH)_{\pi_1(i),\pi_2(j)} = \mH_{i,j}$.

The graph isomorphism and poset-graph isomorphism of hypergraph are defined as follows.
\begin{definition}\label{def:hypergiso}
Hypergraphs $H, H'$ are isomorphic iff there exists $\pi\in \Pi_n\times \Pi_m$, $\pi(H)=H'$. Given node posets $S$ in $H$ and $S'$ in $H'$, $(S, H), (S',H')$ are isomorphic iff there exists $\pi=(\pi_1,\pi_2)\in \Pi_n\times \Pi_m$, $(\pi_1(S),\pi(H))=(S', H')$.
\end{definition}

We can define labeling trick for hypergraph similar to that of graph from scratch. However, converting the hypergraph problem to a graph problem is more convenient. We formalize the known convertion~\citep{IncidenceGraphIntro} as follows.
\begin{definition} 
\textbf{(Incidence graph)} Given a hypergraph $H=(V,E, \mH, \tX^V, \tX^E)$, $V=\{1,2,...,n\}$, $E=\{1,2,...,m\}$, $\mH \in \{0,1\}^{n\times m}$, $\tX^{V}\in\sR^{n\times d}$,$\tX^E\in \sR^{m\times d}$, its incidence graph is $IG_H=(V_H, E_H, \tA)$, where the node set $V_H=\{1,2,...,n,n\!+\!1,...,n\!+\!m\}$, edge set $E_H=\{(i,j)\mid i\in V, j\in E, \mH_{i,j}=1\}$, adjacency tensor $\tA\in \sR^{(n+m)\times (n+m)\times (d+1)}$. For all $i\in V, j\in E$, $\tA_{i, i, :d}=\tX^V_{i,:}$, $\tA_{i, i, d+1}=\tX^V_{i,:}$,  $\tA_{n+j, n+j, :d}=\tX^E_{j,:}$,  $\tA_{i, n+j, d+1}=\mH_{i,j}$. All other elements in $\tA$ are $0$.
\end{definition}
The incidence graph $IG_H$ considers $H$'s nodes and hyperedges both as its nodes. Two nodes in $IG_H$ are connected iff one is a node and the other is a hyperedge containing it in $H$.

The incidence graph contains all information in the hypergraph. Hypergraph isomorphism and poset-hypergraph isomorphism are equivalent to the graph isomorphism and poset-graph isomorphism in the corresponding incidence graphs.
\begin{theorem}\label{thm:hypergiso2giso}
Given node posets $S$ in hypergraph $H$, $S'$ in hypergraph $H'$, $(S, H)\simeq (S',H')$ iff $(S, IG_{H})\simeq (S', IG_{H'})$.
\end{theorem}

Therefore, a hypergraph task can be converted to a graph task. Labeling tricks can be extended to hypergraph by using them on the corresponding incidence graph. 
\begin{corollary}\label{col:hypergexpressivity}
Given an \text{NME GNN}, and an injective aggregation function $\text{AGG}$, for any $S,H, S',H'$, let $\tA, \tA'$ denote the adjacency tensors of graphs $IG_H, IG_{H'}$ respectively. Then $\text{GNN}(S,\tA^{(S)}) = \text{GNN}(S',\tA'^{(S')}) \Leftrightarrow (S,H) \!\simeq\! (S',H') \nonumber$.
\end{corollary}
With NME GNN, set labeling trick can still produce structural representations on hypergraph. This enables us to boost the representation power of hyperedge prediction tasks.

\section{Related work}

There is emerging interest in recent study of graph neural networks' expressivity. \citet{xu2018powerful} and \citet{kWL} first show that the 1-WL test bounds the discriminating power of GNNs performing neighbor aggregation. Many works have since been proposed to increase the power of GNNs by simulating higher-order WL tests~\citep{kWL,PPGN,chen2019equivalence, FWLGNN}, approximating permutation equivariant functions~\citep{IGN, IGNexpressivity, PPGN, FrameAveraging, LocalRelPool}, , encoding subgraphs~\citep{SubgraphGNN, NGNN, KPGNN}, utilizing graph spectral features~\citep{SpectralGT,Sign_basis_invariant}, etc. However, most previous works focus on improving GNN's whole-graph representation power. Little work has been done to analyze GNN's substructure representation power. 
\citet{Srinivasan2020On} first formally studied the difference between structural representations of nodes and links. Although showing that structural node representations of GNNs cannot perform link prediction, their way to learn structural link representations is to give up GNNs and instead use Monte Carlo samples of node embeddings learned by network embedding methods. In this paper, we show that GNNs combined with labeling tricks can also learn structural link representations, which reassures using GNNs for link prediction. 

Many works have implicitly assumed that if a model can learn node representations well, then combining the pairwise node representations can also lead to good node set (for example link) representations~\citep{grover2016node2vec,kipf2016variational,hamilton2017inductive}. However, we argue in this paper that simply aggregating node representations fails to discriminate a large number of non-isomorphic node sets (links), and with labeling trick the aggregation of structural node representations leads to structural representations. 


\citet{li2020distance} proposed distance encoding (DE), whose implementations based on $S$-discriminating distances can be shown to be specific labeling tricks. \citet{you2019position} also noticed that structural node representations of GNNs cannot capture the dependence (in particular distance) between nodes. To learn position-aware node embeddings, they propose P-GNN, which randomly chooses some anchor nodes and aggregates messages only from the anchor nodes. In P-GNN, nodes with similar distances to the anchor nodes, instead of nodes with similar neighborhoods, have similar embeddings. Thus, P-GNN cannot learn structural node/link representations. P-GNN also cannot scale to large datasets.

Finally, although labeling trick is formally defined in our conference paper~\citep{zhang2021labeling}, various forms of specific labeling tricks have already been used in previous works. To our best knowledge, SEAL~\citep{SEAL} proposes to add shortest path distance to target node to each node's feature, which is designed to improve GNN's link prediction power. To our best knowledge, it is the first labeling trick. It is later adopted in the completion of inductive knowledge graphs~\citep{teru2020inductive} and matrix completion~\citep{Zhang2020Inductive}, and is generalized to DE~\citep{li2020distance} and GLASS~\citep{GLASS}, which works for the cases $|S|>2$. \citet{wan2021principled} use labeling trick for hyperedge prediction. Besides these set labeling tricks, some labeling methods similar to the subset labeling trick also exist in existing works. ID-GNN~\citep{you2021identity} and NBFNet~\citep{NBFNet} both use a mechanism equivalent to the one head routine of subset labeling trick. RWL~\citep{RWL} further generalize these methods to a general framework similar to our subset labeling trick with subset size $=1$ and connects its expressivity with logical boolean classifier.

\section{Experiments}

Our experiments include various multi-node representation learning tasks: undirected link prediction, directed link prediction, hyperlink prediction, and subgraph prediction. Labeling trick boosts GNNs on all these tasks. 
All metrics in this section are the higher the better. Datasets are detailed in Appendix~\ref{detaileddatasets}. Our code is available at \url{https://github.com/GraphPKU/LabelingTrick}. In all experiments, we use GNNs without labeling trick (NO) for ablation. 

\subsection{Undirected link prediction}

In this section, we use a two-node task, link prediction, to empirically validate the effectiveness of set and subset labeling trick. 

Following the setting in SEAL~\citep{SEAL}, we use eight datasets: USAir, NS, PB, Yeast, C.ele, Power, Router, and E.coli. These datasets are relatively small. So we additionally use four large datasets in Open Graph Benchmark (OGB)~\citep{hu2020open}: \texttt{ogbl-ppa}, \texttt{ogbl-collab},  \texttt{ogbl-ddi}, \texttt{ogbl-citation2}. To facilitate the comparison, we use the same metrics, including auroc, Hits@$K$, and MRR, as in previous works. 

We use the following baselines for comparison. We use $4$ non-GNN methods: CN (Common-Neighbor), AA (Adamic-Adar), MF (matrix factorization) and Node2vec~\citep{grover2016node2vec}. CN and AA are two simple link prediction heuristics based on counting common neighbors. MF uses free-parameter node embeddings trained end-to-end as the node representations. Two set labeling trick methods are used: ZO and SEAL. ZO uses the zero-one labeling trick, and SEAL uses the DRNL labeling trick~\citep{SEAL}. Three subset labeling trick methods are compared: subset zero-one labeling trick with subset pooling (ZO-S), subset distance encoding labeling trick with subset pooling (DE-S), subset zero-one labeling trick with one-head routine (ZO-OS). 

\noindent\textbf{Results and discussion.} We present the main results in Table~\ref{tab:unlink}. Compared with all non-GNN methods, vanilla 1-WL-GNN with no labeling trick (NO) gets lower auroc on almost all datasets. However, with labeling trick or subset labeling trick, 1-WL-GNN can outperform the baselines on almost all datasets. ZO, SEAL use set labeling trick and outperform non-GNN methods by $4\%$ and $9\%$ respectively on average. The performance difference between ZO and SEAL illustrates that labeling trick implementation can still affect the expressivity of 1-WL-GNN. However, even the simplest labeling trick can still boost 1-WL-GNNs by $6\%$. Subset($1$) labeling trick ZO-S and DE-S also achieve $9\%$ and $11\%$ score increase on average. Compared with ZO, though ZO-S also uses only the target set identity information, it distinguishes nodes in the target node set and achieves up to $5\%$ performance increase on average, which verifies the usefulness of subset labeling trick. Last but not least, though subset labeling trick with one-head routine (ZO-OS) loses permutation invariance compared with subset pooling routine (ZO-S), it still achieves outstanding performance and even outperforms ZO-S on 4/8 datasets.

\begin{table}[t] %
    \centering
\setlength{\tabcolsep}{0.3mm}
\small{
    \begin{tabular}{ccccccccc}
    \toprule
        ~ & USAir & NS & PB & Yeast & Cele & Power & Router & Ecoli \\ \midrule
CN & $93.80_{\pm1.22}$ & $94.42_{\pm0.95}$ & $92.04_{\pm0.35}$ & $89.37_{\pm0.61}$ & $85.13_{\pm1.61}$ & $58.80_{\pm0.88}$ & $56.43_{\pm0.52}$ & $93.71_{\pm0.39}$ \\ 
AA & $95.06_{\pm1.03}$ & $94.45_{\pm0.93}$ & $92.36_{\pm0.34}$ & $89.43_{\pm0.62}$ & $\underline{86.95_{\pm1.40}}$ & $58.79_{\pm0.88}$ & $56.43_{\pm0.51}$ & $95.36_{\pm0.34}$ \\ 
NV & $91.44_{\pm1.78}$ & $91.52_{\pm1.28}$ & $85.79_{\pm0.78}$ & $93.67_{\pm0.46}$ & $84.11_{\pm1.27}$ & $76.22_{\pm0.92}$ & $65.46_{\pm0.86}$ & $90.82_{\pm1.49}$ \\ 
MF & $94.08_{\pm0.80}$ & $74.55_{\pm4.34}$ & $94.30_{\pm0.53}$ & $90.28_{\pm0.69}$ & $85.90_{\pm1.74}$ & $50.63_{\pm1.10}$ & $78.03_{\pm1.63}$ & $93.76_{\pm0.56}$ \\
\midrule
NO & $89.04_{\pm2.14}$ & $74.10_{\pm2.62}$ & $90.87_{\pm0.56}$ & $83.04_{\pm0.93}$ & $73.25_{\pm1.67}$ & $65.89_{\pm1.65}$ & $92.47_{\pm0.76}$ & $93.27_{\pm0.49}$ \\ 
\midrule
ZO & $94.08_{\pm1.43}$ & $95.60_{\pm0.93}$ & $91.82_{\pm1.26}$ & $94.69_{\pm0.45}$ & $74.94_{\pm2.01}$ & $73.85_{\pm1.37}$ & $93.21_{\pm0.66}$ & $92.09_{\pm0.67}$ \\ 
SEAL & $\mathbf{97.09_{\pm0.70}}$ & $97.71_{\pm0.93}$ & $\mathbf{95.01_{\pm0.34}}$ & $97.20_{\pm0.64}$ & $86.54_{\pm2.04}$ & $84.18_{\pm1.82}$ & $\underline{95.68_{\pm1.22}}$ & $97.22_{\pm0.28}$ \\
\midrule
ZO-S & $\underline{96.15_{\pm1.06}}$ & $\underline{98.10_{\pm0.67}}$ & $94.15_{\pm0.50}$ & $97.41_{\pm0.37}$ & $86.31_{\pm1.80}$ & $78.31_{\pm0.91}$ & $94.52_{\pm0.72}$ & $\underline{97.48_{\pm0.23}}$ \\ 
DE-S & $94.97_{\pm0.61}$ & $\mathbf{99.29_{\pm0.14}}$ & $\underline{94.44_{\pm0.52}}$ & $\mathbf{98.17_{\pm0.41}}$ & $85.95_{\pm0.36}$ & $\mathbf{94.16_{\pm0.14}}$ & $\mathbf{99.33_{\pm0.09}}$ & $\mathbf{98.91_{\pm0.08}}$ \\ 
ZO-OS & $94.62_{\pm0.63}$ & $97.42_{\pm0.49}$ & $94.36_{\pm0.26}$ & $\underline{97.46_{\pm0.06}}$ & $\mathbf{88.04_{\pm0.52}}$ & $\underline{84.95_{\pm0.30}}$ & $93.77_{\pm0.20}$ & $95.53_{\pm0.62}$ \\ \bottomrule
    \end{tabular}
}

    \caption{Results on undirected link prediction task: auroc (\%) $\pm$ standard deviation.}\label{tab:unlink}
\end{table}

We also conduct experiments on some larger datasets as shown in Table~\ref{tab:unlink2}. GNN augmented by labeling tricks achieves the best performance on all datasets. 

\begin{table}[t]
    \centering
    \centering
    \begin{tabular}{ccccc}
    \hline
    \toprule
        Dataset & collab & ddi & citation2 & ppa \\ \midrule
        metrics& Hits@50 & Hits@20 & MRR & Hits@100\\
        \midrule
        NO & $44.75_{\pm 1.07}$ & $\underline{37.07_{\pm 5.07}}$ & $\underline{84.74_{\pm 0.21}}$ & $18.67_{\pm 1.32}$ \\ 
        \midrule
        ZO & $53.29_{\pm 0.23}$ & $23.90_{\pm 0.75}$ & $78.50_{\pm 1.08}$ & $37.75_{\pm 3.42}$ \\ 
        SEAL & $\mathbf{54.71_{\pm 0.49}}$ & $30.56_{\pm 3.86}$ & $\mathbf{87.67_{\pm 0.32}}$ & $\mathbf{48.80_{\pm 3.16}}$ \\ 
        \midrule
        ZO-OS & $49.17_{\pm 3.29}$ & $\mathbf{41.24_{\pm 1.49}}$ & $82.85_{\pm 0.43}$ & $\underline{43.27_{\pm 1.19}}$ \\
        ZO-S & $\underline{54.69_{\pm 0.51}}$ & $29.27_{\pm 0.53}$ & $82.45_{\pm 0.62}$ & $36.04_{\pm 4.50}$ \\ 
    \bottomrule
    \end{tabular}
    \caption{Results on undirected link prediction task.}\label{tab:unlink2}
\end{table}

\subsection{Directed link prediction tasks}

To illustrate the necessity of introducing partial order to labeling trick, we compare set labeling trick and poset labeling trick on the directed link prediction task. Following previous work~\citep{PyGsd}, we use six directed graph datasets, namely Cornell, Texas, Wisconsin, CoraML, Citeseer, and Telegram. Our baselines includes previous state-of-the-art GNNs for directed graph, including DGCN~\citep{DGCN}, DiGCN and DiGCNIB~\citep{DIGCN}, and MagNet~\citep{MagNet}. Our models include NO (vanilla 1-WL-GNN), PL (poset labeling trick which labels the source node as $1$, target node as $2$, other nodes as $0$), ZO (zero-one labeling trick).

The results are shown in Table~\ref{tab:directed}. The existing state-of-the-art method MAGNet~\citep{zhang2021magnet} outperforms 1-WL-GNN by $0.25\%$ on average. However, 1-WL-GNN with labeling trick outperforms all baselines. Moreover, poset labeling trick (PL) achieves $2\%$ performance gain compared with the set labeling trick (ZO). These results validate the power of poset labeling trick and show that modeling partial order relation is critical for some tasks.

\begin{table}[t]
    \centering
\small{
    \begin{tabular}{ccccccc}
    \toprule
 & Cornell & Texas & Wisconsin &  CoraML &  CiteSeer & Telegram \\ \midrule
DGCN & $70.4_{\pm 9}$ & $69.7_{\pm 6}$ & $69.8_{\pm 6}$ & $77.2_{\pm 1}$ & $71.2_{\pm 2}$ & $86.4_{\pm 1}$ \\ 
DiGCN & $69.3_{\pm 7}$ & $69.7_{\pm 7}$ & $66.2_{\pm 7}$ & $75.1_{\pm 1}$ & $73.0_{\pm 1}$ & $78.2_{\pm 1}$ \\ 
DiGCNIB & $65.7_{\pm 3}$ & $63.7_{\pm 6}$ & $67.6_{\pm 6}$ & $75.9_{\pm 4}$ & $73.7_{\pm 2}$ & $79.8_{\pm 2}$ \\ 
MagNet & $70.4_{\pm 8}$ & $73.1_{\pm 7}$ & $70.4_{\pm 7}$ & $77.3_{\pm 1}$ & $71.8_{\pm 1}$ & $\underline{86.7_{\pm 1}}$ \\ \midrule
NO & $67.5_{\pm 7}$ & $72.0_{\pm 4}$ & $71.6_{\pm 5}$ & $78.7_{\pm 1}$ & $74.3_{\pm 1}$ & $84.1_{\pm 1}$ \\ \midrule
PL & $\mathbf{71.5_{\pm 4}}$ & $\mathbf{78.0_{\pm 4}}$ & $\mathbf{79.0_{\pm 3}}$ & $\mathbf{82.2_{\pm 1}}$ & $\underline{79.8_{\pm 1}}$ & $\mathbf{91.7_{\pm 1}}$ \\ 
ZO & $\underline{71.1_{\pm 5}}$ & $\underline{77.6_{\pm 5}}$ & $\underline{74.0_{\pm 2}}$ & $\underline{79.4_{\pm 1}}$ & $\mathbf{79.9_{\pm 1}}$ & $85.6_{\pm 1}$ \\ 
\bottomrule
    \end{tabular}
    }

    \caption{Results on directed link prediction tasks: accuracy (\%) $\pm$ standard deviation.}\label{tab:directed}
\end{table}

\begin{table}[t]
    \centering
    \setlength{\tabcolsep}{2mm}
    \begin{tabular}{cccccccc}
    \toprule
~ & NDC-c & NDC-s & tags-m & tags-a & email-En & email-EU & congress\\\midrule
ceGCN & $61.4_{\pm0.5}$ & $42.1_{\pm1.4}$ & $59.9_{\pm0.9}$ & $54.5_{\pm0.5}$ & $61.8_{\pm3.2}$ & $66.4_{\pm0.3}$ & $41.2_{\pm0.3}$\\
ceSAGE & $65.7_{\pm2.0}$ & $47.9_{\pm0.7}$ & $63.5_{\pm0.3}$ & $59.7_{\pm0.7}$ & $59.4_{\pm4.6}$ & $65.1_{\pm1.9}$ & $53.0_{\pm5.5}$\\
seRGCN & $67.6_{\pm4.9}$ & $52.5_{\pm0.6}$ & $57.2_{\pm0.3}$ & $54.5_{\pm0.6}$ & $59.9_{\pm4.0}$ & $66.1_{\pm0.6}$ & $54.4_{\pm0.4}$\\
FS & $\underline{76.8_{\pm0.4}}$ & $51.2_{\pm3.2}$ & $64.2_{\pm0.6}$ & $60.5_{\pm0.2}$ & $68.5_{\pm1.6}$ & $68.7_{\pm0.2}$ & $56.6_{\pm1.1}$\\ \midrule
NO & $60.2_{\pm2.3}$ & $45.6_{\pm0.8}$ & $56.6_{\pm1.4}$ & $56.5_{\pm1.8}$ & $56.9_{\pm1.7}$ & $57.2_{\pm0.9}$ & $54.1_{\pm0.5}$\\ \midrule
ZO & $\mathbf{82.5_{\pm1.3}}$ & $\mathbf{63.6_{\pm1.5}}$ & $\mathbf{71.4_{\pm0.5}}$ & $\mathbf{70.4_{\pm0.8}}$ & $\underline{66.1_{\pm1.2}}$ & $\underline{72.1_{\pm1.1}}$ & $\mathbf{65.1_{\pm0.2}}$\\
ZO-S & $75.8_{\pm0.7}$ & $\underline{62.2_{\pm1.2}}$ & $\underline{71.0_{\pm0.4}}$ & $\underline{69.6_{\pm0.7}}$ & $\mathbf{67.7_{\pm1.8}}$ & $\mathbf{73.3_{\pm0.5}}$ & $\underline{64.2_{\pm0.3}}$\\ \bottomrule
    \end{tabular}
    
    \caption{Results on hyperedge prediction tasks: f1-score ($\%$) $\pm$ standard deviation.}
 \end{table}

\begin{table}[t]
    \centering
    \begin{tabular}{ccccccccc}
    \toprule
        Method & density & coreness & cutratio &  ppi-bp & hpo-metab & hpo-neuro & em-user\\ \midrule
        SubGNN & $91.9_{\pm 0.6}$ & $65.9_{\pm 3.1}$ & $62.9_{\pm 1.3}$& ${59.9_{\pm 0.8}}$ & ${53.7_{\pm 0.8}}$ & ${64.4_{\pm 0.6}}$ & $81.6_{\pm 1.3}$\\
        Sub2Vec & $45.9_{\pm 1.2}$ & $36.0_{\pm 1.9}$ &$35.4_{\pm 1.4}$& ${38.8_{\pm 0.1}}$ & ${47.2_{\pm 1.0}}$ & ${61.8_{\pm 0.3}}$ & $77.9_{\pm 1.3}$\\
        \midrule
        NO &  $47.8_{\pm 2.9}$& $47.8_{\pm 5.3}$& $81.4_{\pm 1.5}$ & ${61.3_{\pm 0.9}}$ & ${59.7_{\pm 1.2}}$ & ${66.8_{\pm 0.7}}$ & $84.7_{\pm 2.1}$\\
        \midrule
        ZO   & $\mathbf{98.4_{\pm 1.2}}$ & $\mathbf{87.3_{\pm 15.0}}$ & $\mathbf{93.0_{\pm 1.3}}$ & $\mathbf{61.9_{\pm 0.7}}$ & $\mathbf{61.4_{\pm 0.5}}$ & $\mathbf{68.5_{\pm 0.5}}$ & $\mathbf{88.8_{\pm 0.6}}$\\ 
        ZO-S & \underline{$94.3_{\pm 6.9}$} & \underline{$75.8_{\pm 7.0}$} & \underline{$85.6_{\pm 2.5}$}&  \underline{${61.7_{\pm 0.4}}$}&\underline{$60.4_{\pm1.1}$} & \underline{$67.4_{\pm1.3}$} & \underline{$86.3_{\pm 2.5}$}\\
    \bottomrule
    \end{tabular}
    \caption{Results on subgraph tasks: f1-score ($\%$) $\pm$ standard deviation.}\label{tab:subg}
\end{table}

\subsection{Hyperedge prediction task}

We use the datasets and baselines in \citep{familyset}. Our datasets includes two drug networks (NDC-c, NDC-s), two forum networks (tags-m, tags-a), two email networks (email-En, email-Eu), and a network of congress members (congress). We use four GNNs designed for hypergraph as baselines, including ceGCN, ceSAGE, seRGCN, and FS (family set)~\citep{familyset}. Our models include ZO (zero-one labeling trick), ZO-S (subset($1$) labeling trick with subset pooling), NO (vanilla 1-WL-GNN).
 ZO and ZO-S outperform all other methods significantly.

\subsection{Subgraph prediction task}

We use the datasets and baselines in \citep{subgnn}. We use three synthetic datasets, namely density, coreness, and cutratio, and four real-world datasets, namely ppi-bp, hpo-metab, hpo-neuro, em-user. SubGNN~\citep{subgnn} and Sub2Vec~\citep{sub2vec} are models designed for subgraph. Our models include ZO (zero-one labeling trick, results on ppi-bp, hpo-metab, hpo-neuro, em-user are from \citet{GLASS}), ZO-S (subset labeling trick), and NO (vanilla 1-WL-GNN without labeling trick, results on ppi-bp, hpo-metab, hpo-neuro, em-user are from \citet{GLASS}). Compared with NO, Labeling tricks boost vanilla 1-WL-GNN significantly. Moreover, vanilla GNN augmented by labeling trick also outperforms GNN designed for subgraph on all datasets. Moreover, ZO outperforms ZO-S, which illustates that subset labeling tricks, while ZO can capture high-order relations better as shown in Section~\ref{sec:set_vs_subset}.

\section{Conclusions}
In this paper, we proposed a theory of using GNNs for multi-node representation learning. We first pointed out the key limitation of a common practice in previous works that directly aggregates node representations as a node-set representation. To address the problem, we proposed set labeling trick which gives target nodes distinct labels in a permutation equivariant way and characterized its expressive power. We further extended set labeling trick to poset and subset labeling trick, as well as extending graph to hypergraph. Our theory thoroughly discusses different variants and scenarios of using labeling trick to boost vanilla GNNs, and provides a solid foundation for future researchers to develop novel labeling tricks.

\section*{Acknowledgments}
This work is supported by the National Key R\&D Program of China (2022ZD0160300) and National Natural Science Foundation of China (62276003).


\newpage

\appendix

\section{Proofs}
\subsection{Proof of Proposition~\ref{prop::gsiso2giso} and Proposition~\ref{prop::psgiso2giso}}
For Proposition~\ref{prop::gsiso2giso},
\begin{align*}
    (S, \tA)\simeq (S', \tA') &\Leftrightarrow \exists\pi\in \Pi_n, \pi(S)=S', \pi(\tA) = \tA'\\
    &\Leftrightarrow \exists\pi\in \Pi_n, \pi(\tL(S, \tA))=\tL(S', \tA'), \pi(\tA)=\tA'\\
    &\Leftrightarrow \exists\pi\in \Pi_n, \pi(\tA^{(S)})={\tA'}^{(S')}
\end{align*}

For Proposition~\ref{prop::psgiso2giso}, we can simply replace set $S$ above with poset.

\subsection{Proof of Theorem~\ref{thm:labeltrick}}

Following \citet{zhang2021labeling}, we restate Theorem~\ref{thm:labeltrick}: Given an \text{NME GNN} and an injective set aggregation function \text{AGG}, for any $S,\tA, S',\tA'$, $\text{GNN}(S,\tA^{(S)}) = \text{GNN}(S',\tA'^{(S')}) \Leftrightarrow (S,\tA) \!\simeq\! (S',\tA') \nonumber$, where $\text{GNN}(S,\tA^{(S)}) := \text{AGG}(\{\text{GNN}(i,\tA^{(S)}) | i\in S\})$.

\begin{proof}

We need to show $\text{AGG}(\{\text{GNN}(i,\tA^{(S)})|i\in S\})= \text{AGG}(\{\text{GNN}(i,\tA'^{(S')}|i\in S'\})$ $\Leftrightarrow$ $(S,\tA) \simeq (S',\tA')$.

To prove $\Rightarrow$, we notice that with an injective AGG, 
\begin{align}
    &\text{AGG}(\{\text{GNN}(i,\tA^{(S)}))|i\in S\}) = \text{AGG}(\{\text{GNN}(i,\tA'^{(S')}))|i\in S'\}) \nonumber\\
\Longrightarrow ~& \exists~v_1\in S, v_2\in S',~\text{such that}~\text{GNN}(v_1,\tA^{(S)}) = \text{GNN}(v_2,\tA'^{(S')}) \\
\Longrightarrow ~& (v_1,\tA^{(S)}) \simeq (v_2,\tA'^{(S')}) ~~~~\text{(because GNN is node-most-expressive)}\\
\Longrightarrow ~&  \exists~\pi \in \Pi_n,~\text{such that}~ v_1 = \pi(v_2), \tA^{(S)} = \pi(\tA'^{(S')}).\label{eq:3}
\end{align}

Remember $\tA^{(S)}$ is constructed by stacking $\tA$ and $\tL(S, \tA)$ in the third dimension, where $\tL(S,\tA)$ is a tensor satisfying: $\forall \pi \in \Pi_n$, \text{ (1)} $\tL(S,\tA) = \pi(\tL(S', \tA')) \Rightarrow S = \pi(S')$, and \text{ (2)} $S = \pi(S'), \tA = \pi(\tA') \Rightarrow \tL(S,\tA) = \pi(\tL(S',\tA'))$. With $\tA^{(S)} = \pi(\tA'^{(S')})$, we have both 
\begin{align*}
    \tA = \pi(\tA'),~\tL(S,\tA) = \pi(\tL(S', \tA')).\nonumber
\end{align*}
Because $\tL(S,\tA) = \pi(\tL(S',\tA')) \Rightarrow S = \pi(S')$, continuing from Equation~(\ref{eq:3}), we have
\begin{align*}
    &\text{AGG}(\{\text{GNN}(i,\tA^{(S)})|i\in S\}) = \text{AGG}(\{\text{GNN}(i,\tA'^{(S')})|i\in S'\}) \nonumber\\
\Longrightarrow ~& \exists~\pi \in \Pi_n,~\text{such that}~ \tA = \pi(\tA'),~\tL(S,\tA) = \pi(\tL(S',\tA')) \\ 
\Longrightarrow ~& \exists~\pi \in \Pi_n,~\text{such that}~ \tA = \pi(\tA'),~S = \pi(S')\\ 
\Longrightarrow ~& (S,\tA) \simeq (S',\tA').
\end{align*}

Now we prove $\Leftarrow$. Because $S = \pi(S'), \tA = \pi(\tA') \Rightarrow \tL(S,\tA) = \pi(\tL(S',\tA'))$, we have:
\begin{align*}
    &(S,\tA) \simeq (S',\tA') \nonumber\\
\Longrightarrow ~& \exists~\pi \in \Pi_n,~\text{such that}~ S = \pi(S'), \tA = \pi(\tA') \\
\Longrightarrow ~& \exists~\pi \in \Pi_n,~\text{such that}~ S = \pi(S'), \tA = \pi(\tA'), \tL(S,\tA) = \pi(\tL(S',\tA')) \\
\Longrightarrow ~& \exists~\pi \in \Pi_n,~\text{such that}~ S = \pi(S'), \tA^{(S)} = \pi(\tA'^{(S')}) \\
\Longrightarrow ~& \exists~\pi \in \Pi_n,~\text{such that}~ \forall v_2 \in S', v_1 = \pi(v_2)\in S, \text{GNN}(v_1,\tA^{(S)}) = \text{GNN}(v_2,\tA'^{(S')}) \\
\Longrightarrow ~&  \text{AGG}(\{\text{GNN}(v_1,\tA^{(S)})|v_1\in S\}) = \text{AGG}(\{\text{GNN}(v_2,\tA'^{(S')})|v_2\in S'\}),
\end{align*}
which concludes the proof.
\end{proof}
\subsection{Proof of Theorem~\ref{thm:link-aso-num} and Theorem~\ref{thm:PlabelingBoostWLGNN}}\label{app:proofBoostWLLink}
Following \citet{zhang2021labeling}, as an $h$-layer 1-WL-GNN only encodes an $h$-hop neighbors for each node, we define locally $h$-isomorphism. 

\begin{definition}
    For all $S,\tA,S',\tA'$, $(S,\tA)$ and $(S',\tA')$ are locally $h$-isomorphic iff $(S, \tA_{S,h})\simeq (S', \tA_{S', h})$, where $\tA_{S,h}$ means the subgraph of $\tA$ induced by the node set $\{v\in V|\exists u\in S, d_{sp}(u,v,\tA)\le h\}$, and $d_{sp}(u,v,\tA)$ means the shortest path distance between node $u, v$ in graph $\tA$.
\end{definition}

We restate Theorem~\ref{thm:link-aso-num}(Theorem~\ref{thm:PlabelingBoostWLGNN}): In any non-attributed graph with $n$ nodes, if the degree of each node in the graph is between $1$ and $\left((1-\epsilon)\log n\right)^{1/(2h+2)}$ for any constant $\epsilon>0$, then there exists $\omega(n^{2\epsilon})$ many pairs of non-isomorphic links $(u,w), (v,w)$ such that an $h$-layer 1-WL-GNN gives $u,v$ the same representation, while with zero-one labeling trick (subset zero-one labeling trick) the 1-WL-GNN gives $u,v$ different representations. These two theorems can be proved together because the special cases we build can be solved by both of them. 

\begin{proof}

Our proof has two steps. First, we would like to show that there are $\omega(n^{\epsilon})$ nodes that are locally $h$-isomorphic to each other. Then, we prove that among these nodes, there are at least $\omega(n^{2\epsilon})$ pairs of nodes such that there exists another node constructing locally $h$ non-isomorphic links with either of the two nodes in each node pair. 

\noindent\textbf{Step 1.} Consider an arbitrary node $v$ and denote the node set induced by the nodes that are at most $h$-hop away from $v$ as $G_v^{(h)}$ (the $h$-hop enclosing subgraph of $v$). As each node is with degree $d \le \big((1-\epsilon)\log n\big)^{1/(2h+2)}$, then the number of nodes in $G_v^{(h)}$, denoted by $|V(G_v^{(h)})|$, satisfies 
\begin{align*}
 |V(G_v^{(h)})| \leq \sum_{i=0}^{h} d^i \le d^{h+1} = \big((1-\epsilon)\log n\big)^{1/2}.
\end{align*}
We set $K=\max_{v\in V} |V(G_v^{(h)})| \le \big((1-\epsilon)\log n\big)^{1/2}$. 

Now we expand subgraphs $G_v^{(h)}$ to $\bar{G}_v^{(h)}$ by adding $K-|V(G_v^{(h)})|$ independent nodes for each node $v\in V$. Then, all $\bar{G}_v^{(h)}$ have the same number of nodes, which is $K$, though they may not be connected graphs. Next, we consider the number of non-isomorphic graphs over $K$ nodes. Actually, the number of non-isomorphic graph structures over $K$ nodes is bounded by
\begin{equation*}
    2^{K \choose 2} \le 2^{(1-\epsilon)\log n}=n^{1-\epsilon}.
\end{equation*}

Therefore, due to the pigeonhole principle, there exist $\omega(n/n^{1-\epsilon}) = \omega(n^\epsilon)$ many nodes $v$ whose $\bar{G}_v^{(h)}$ are isomorphic to each other. Denote the set of these nodes as $V_{iso}$, which consist of nodes that are all locally $h$-isomorphic to each other. 

\noindent\textbf{Step 2.} Let us partition $V_{iso}=\cup_{i=1}^q V_i$ so that for all $i\in \{1,2,...,q\}$, nodes in $V_i$ share the same first-hop neighbor sets. Note that all nodes in each $V_i$ share the same neighbors, so $|V_i|$ is no more than maximum degree $((1-\epsilon)\log n)^{1/(2h+2)}< n^\epsilon$ when $\epsilon>\frac{1}{(2h+2)\log n}(\log \log n) $. Then, consider any pair of nodes $u,v$ such that $u,v$ are from different $V_i$'s. Since $u,v$ share identical $h$-hop neighborhood structures, an $h$-layer 1-WL-GNN will give them the same representation. Then, we may pick one $w\in N(u)-N(v)$ (If $w$ does not exists, then $N(u)-N(v)=\empty$, so $N(v)-N(u)\neq \emptyset$ because of the definition of $V_i$. We can simply exchange $u$ and $v$). As $w$ is $u$'s first-hop neighbor and is not $v$'s first-hop neighbor, $(u,w)$ and $(v,w)$ are not isomorphic. With labeling trick, the $h$-layer 1-WL-GNN will give $u,v$ different representations immediately after the first message passing round due to $w$'s distinct label. Therefore, we know such a $(u,w), (v,w)$ pair is exactly what we want.

Based on the partition $V_{iso}$, we know the number of such non-isomorphic link pairs $(u, w)$ and $(v,w)$ is at least:
\begin{align}\label{eq:apd1}
    Y \geq \sum_{i,j=1, i< j}^q |V_i||V_j| = \frac{1}{2}\left[(\sum_{i=1}^q|V_i|)^2 - \sum_{i=1}^q|V_i|^2\right].
\end{align}

Because of the definitions of the partition,  $\sum_{i=1}^q|V_i| = |V_{iso}|=\omega(n^\epsilon)$ and the size of each $V_i$ satisfies
\begin{align*}
    1\leq  |V_i| \leq d_w \le \big((1-\epsilon)\log n\big)^{1/(2h+2)},
\end{align*}
where $w$ is one of the common first-hop neighbors shared by all nodes in $V_i$ and $d_w$ is its degree.

By plugging in the range of $|V_i|$, Eq.\ref{eq:apd1} leads to 
\begin{align*}
    Y &\geq\frac{1}{2}[(\sum_{i=1}^q|V_i|)^2 - \sum_{i=1}^q|V_i|(\max_{j\in \{1,2,...,q\}}|V_j|)]\\ 
    &=\frac{1}{2}(\omega(n^{2\epsilon}) - \omega(n^\epsilon)\mathcal{O}\Big(\big((1-\epsilon)\log n\big)^{1/(2h+2)}\Big) \\
    &=\omega(n^{2\epsilon}),
\end{align*}
which concludes the proof. 
\end{proof}

\subsection{Proof of Theorem~\ref{thm:BoostWLGNNsubg}}\label{app:proofBoostWLSubg}
\begin{proof}
This proof shares the same first step as Appendix~\ref{app:proofBoostWLLink}.

\noindent\textbf{Step 2.} Let us partition $V_{iso}=\bigcup_{i=1}^q V_i$, nodes in each $V_i$ share the same one-hop neighbor. Consider two nodes $u\in V_i,v\in V_j,i\neq j$. There exists a node $w\in N(u),w\notin N(v)$ (If $w$ does not exists, then $N(u)-N(v)=\empty$, so $N(v)-N(u)\neq \emptyset$ because of the definition of $V_i$. We can simply exchange $u$ and $v$). Let $\tilde V_{u, v, w}$ denote $V-\{u,v,w\}-N(v)$. $\tilde V_{u,v,w}\ge n-3-\big((1-\epsilon)\log n\big)^{1/(2h+2)}$. Consider arbitrary subset $V'$ of $\tilde V_{u,v,w}$. Let $\gS_1$ denote the subgraph induced by $V'\bigcup \{u,w\}$, $\gS_2$ denote the subgraph induced by $V'\bigcup \{v,w\}$. Compared with $\gS_2$, $\gS_1$ has the same number of nodes. Moreover, $\gS_2$ has edge between nodes in $V'$ and edges between $V'$ and $w$, while $\gS_1$ further has more edge $(u, w)$ and edges between $V'$ and ${u}$, so the density of $\gS_1$ is higher than $\gS_2$. And 1-WL-GNN with zero-one labeling trick can fit density perfectly (Theorem~1 in \citep{GLASS}), so 1-WL-GNN with labeling trick can distinguish $\gS_1$ and $\gS_2$, while 1-WL-GNNs cannot.

The number of pair $(u, v, w)$ is $w(n^{2\epsilon})$. Therefore, the number of these pairs of subgraphs is bounded by 
\begin{equation*}
w(n^{2\epsilon})2^{n-3-((1-\epsilon)\log n)^{1/(2h+2)}}=w(2^n n^{3\epsilon-1}).
\end{equation*}
\end{proof}

\subsection{Proof of Theorem~\ref{thm:BoostWLGNNonPoset}}
This proof shares the same first step as Appendix~\ref{app:proofBoostWLLink}.

Number of link: the same as the step 2 in Appendix~\ref{app:proofBoostWLLink}.

Number of subgraph: similar to the step 2 in Appendix~\ref{app:proofBoostWLSubg}. Let us partition $V_{iso}=\bigcup_{i=1}^q V_i$, nodes in each $V_i$ share the same one-hop neighbor. Consider two nodes $u\in V_i,v\in V_j,i\neq j$. There exists a node $w\in N(u),w\notin N(v)$. Let $\tilde V_{u, v, w}$ denote $V-\{u,v,w\}-N(u)$. $|V_v|\ge n-3-\big((1-\epsilon)\log n\big)^{1/(2h+2)}$. Consider arbitrary subset $V'$ of $\tilde V_{u,v,w}$ and a partial order $\le_{V'}$. Let $\gS_1$ denote the subgraph induced by poset $\big((V'\bigcup \{u,w\}), \le_{V'}\cup \{(u, a)|a\in sV'\}\cup \{(w, a)|a\in sV'\cup \{u\}\} \big)$, $\gS_2$ denote the subgraph induced by poset $\big(V'\bigcup \{v,w\}, \le_{V'}\cup \{(v, a)|a\in sV'\}\cup \{(w, a)|a\in V'\cup \{v\}\} \big)$. 1-WL-GNN with labeling trick can distinguish $\gS_1$ and $\gS_2$ as the edges between $(u, w)$ and $(v, w)$ are distinct, while 1-WL-GNNs cannot.

The number of pair $(u, v, w)$ is $w(n^{2\epsilon})$. Therefore, the number of these pairs of subgraphs is bounded by 
\begin{equation*}
w(n^{2\epsilon})w(n-3-\big((1-\epsilon)\log n\big)^{1/(2h+2)})!=w\Big(\big((1-\epsilon)n\big)!\Big).
\end{equation*}

\subsection{Proof of Proposition~\ref{prop:cnraaa}}

As shown in Figure~\ref{node_iso}, 1-WL-GNN cannot count common neighbor and thus fail to implement $h$. Now we prove that with zero-one labeling trick, 1-WL-GNN can implement $h$.

Given a graph $\tA$ and a node pair $(i, j)$, let $z_k^{(k)}$ denote the embedding of node $i$ at $k^{\text{th}}$ message passing layer. 
\begin{equation*}
    z_k^{(0)}=\begin{bmatrix}
        1\\
        \delta_{ki}+\delta_{kj}
    \end{bmatrix}.
\end{equation*}
The first dimension is all $1$ (vanilla node feature), and the second dimension is zero-one label. 

The first layer is,
\begin{equation*}
z_k^{(1)}=\begin{bmatrix}
        g_1(a_k^{(1)}[1])\\
        g_2(a_k^{(1)}[1])\\
        a_k^{(1)}[2] > 2
    \end{bmatrix}
\end{equation*}
where $a_k^{(1)} = \sum_{l\in N(k)}z_{l}^{(0)}$, $[1]$ means the first element of vector, and $[2]$ means the second element.

The second layer is
\begin{equation*}
    z_{k}^{(2)} = \begin{bmatrix}
        \sum_{l\in N(k)} z_{k}^{(1)}[3]  z_{k}^{(1)}[2]\\
        \sum_{l\in N(k)} (1-z_{k}^{(1)}[3])  z_{k}^{(1)}[1]\\
    \end{bmatrix}
\end{equation*}

The pooling layer is
\begin{equation*}
    z_{ij}=f(\{z_{i}[2], z_{j}[2]\},\frac{z_{i}[1]+z_{j}[1]}{2})
\end{equation*}

\subsection{Proof of Theorem~\ref{thm:poexpressivity}}
\begin{proof}
    $\Leftarrow$: 
    When $(S, \tA)\simeq (S', \tA')$, there exists a permutation $\pi$, $\pi(S)=S', \pi(\tA)=\tA'$. 
    
    \begin{align}
        \text{GNN}(S,\tA^{(S)}) &= \text{AGG}(\{\text{GNN}(v, \tA^{(S)}|v\in S\})\\
        &= \text{AGG}(\{\text{GNN}(\pi(v), \pi(\tA^{(S)}))|v\in S\})\\
        &= \text{AGG}(\{\text{GNN}(\pi(v), \tA'^{(S')}|v\in S\})\\
        &= \text{AGG}(\{\text{GNN}(v', \tA'^{(S')})|v'\in S'\})\\
        &= \text{GNN}(S',{\tA'}^{(S')})
    \end{align}
    
    $\Rightarrow$: 
    \begin{equation*}
        \text{GNN}(S,\tA^{(S)}) = \text{GNN}(S',{\tA'}^{(S')})\\
    \end{equation*}
    \begin{equation*}
        \text{AGG}(\{\text{GNN}(v, \tA^{(S)})|v\in S\}) 
        = \text{AGG}(\{\text{GNN}(v', {\tA'}^{(S')})|v'\in S'\})
    \end{equation*}
    As $\text{AGG}$ is injective, 
    There exist $v_0\in S, v'_0\in S'$, 
    \begin{equation*}
        \text{GNN}(v_0, \tA^{(S)})=\text{GNN}(v_0', \tA^{(S')})
    \end{equation*}
    As GNN is node most expressive,
    \begin{equation*}
    \exists \pi , \pi(v_0)=v_0', \pi(\tA)=\tA', \pi(\tL(S,\tA))=\tL(S',\tA').
    \end{equation*}
    Therefore, $\pi(\tL(S,\tA))=\tL(S', \tA'))$.
    \end{proof}

\subsection{Proof of Theorem~\ref{thm:k-1plabelingexpressivity}}
\begin{proof}
    $\Leftarrow$: When $(S, \tA)\simeq (S', \tA')$, there exists a permutation $\pi$, $\pi(S)=S', \pi(\tA)=\tA'$. 
    
    \begin{align*}
    \text{AGG}&\left(\left\{\text{AGG} (\{\text{GNN}(u, \tA^{(S-\{v\})})|u\in S\})|v\in S\right\}\right)\\
        &= \text{AGG}(\{\text{AGG}(\{\text{GNN}(\pi(u), \pi(\tA^{(S-\{v\})})|u\in S \})|v\in S\})\\
        &= \text{AGG}(\{\text{AGG}(\{\text{GNN}(\pi(u), \tA'^{(\pi(S)-\{\pi(v)\})}|u\in S\})|v\in S\})\\
        &= \text{AGG}(\{\text{AGG}(\{\text{GNN}(u', \tA'^{(S'-\{v'\})})|u'\in S'\})|v'\in S'\})
    \end{align*}
    
    $\Rightarrow$: 
    \begin{align*}
        \text{AGG}(&\{\text{AGG}(\{\text{GNN}(u, \tA^{(S-\{v\})})|u\in S\})|v\in S\})\\ 
        &= \text{AGG}(\{\text{AGG}(\{\text{GNN}(u', \tA'^{(S'-\{v'\})})|u'\in S' \})|v'\in S'\}).
    \end{align*}
    As $\text{AGG}$ is injective, 
    \begin{equation*}
        \{\text{AGG}(\{\text{GNN}(u, \tA^{\!(S-\{\!v\!\})\!}) |v\in S\})|u\in S\}= \{\{\text{AGG}(\{\text{GNN}(u', \tA^{\!(S'-\{\!v'\!\})\!}) |v'\in S'\})\}.
    \end{equation*}
    There exist $v_0\in S, v'_0\in S'$, 
    \begin{equation*}
        \text{AGG}(\{\text{GNN}(u, \tA^{(S-\{v_0\})})|u\in S \})
        = \text{AGG}(\{\text{GNN}(u', \tA^{(S'-\{v'_0\})}) |u'\in S'\}).
    \end{equation*}
    Similarly, there exists $u_0'\in S'$
    \begin{equation*}
        \text{GNN}(v_0, \tA^{(S-\{v_0\})})= \text{GNN}(u_0', \tA^{(S'-\{v'_0\})}).
    \end{equation*}
    As GNN is node most expressive,
    $$
    \exists \pi , \pi(v_0)=u_0', \pi(\tA)=\tA', \pi(\tL(S-\{v_0\}, \tA))=\tL(S'-\{v'_0\},\tA')).
    $$
    Therefore, $\pi(S-\{v_0\})=S'-\{v'_0\}$. Note that $v_0\notin S-\{v_0\}$, so $u'_0
    =\pi(v_0) \notin S'-\{v'_0\}$, while $u'_0\in S'$, therefore $u'_0=v'_0$.
    
    Therefore, $\pi(S)=S'$, and $\pi(\tA)=\tA'$, so $(S, \tA)\simeq (S', \tA')$. 
\end{proof}

\subsection{Proof of Theorem~\ref{thm:k-1plabelingexpressivity2}}

We prove it by contradiction:
If $\exists v_0\in S, v_0'\in S'$, 
\begin{equation*}
    \text{GNN}(S, \tA^{(S-\{v_0\})})=\text{GNN}(S', {\tA'}^{(S'-\{v_0'\})})
\end{equation*}

Therefore, there exists $u_0\in S,u_0'\in S'$
    \begin{equation*}
        \text{GNN}(v_0, \tA^{(S-\{v_0\})})= \text{GNN}(u_0', \tA^{(S'-\{v'_0\})}).
    \end{equation*}
    As GNN is node most expressive,
    $$
    \exists \pi , \pi(v_0)=u_0', \pi(\tA)=\tA', \pi(\tL(S-\{v_0\}, \tA))=\tL(S'-\{v'_0\},\tA')).
    $$
    Therefore, $\pi(S-\{v_0\})=S'-\{v'_0\}$. Note that $v_0\notin S-\{v_0\}$, so $u'_0
    =\pi(v_0) \notin S'-\{v'_0\}$, while $u'_0\in S'$, therefore $u'_0=v'_0$.
    
    Therefore, $\pi(S)=S'$, and $\pi(\tA)=\tA'$, so $(S, \tA)\simeq (S', \tA')$, which contradicts to that $(S, \tA)\not\simeq (S', \tA')$.

\subsection{Proof of Proposition~\ref{prop:NME GNN-L>PL}}

Figure~\ref{fig:PWL3} provides an example.

\subsection{Proof of Proposition~\ref{prop:WLGNN-L<>PL}}
Figure~\ref{fig:PWL} and Figure~\ref{fig:PWL3} provide example.

\subsection{Proof of Proposition~\ref{prop:pohasse}}

Due to the property 1 in Definition~\ref{def:Lposet}, $\tL(S,\tA)=\pi(\tL(S',\tA'))\Rightarrow S=\pi(S')$. Therefore, for all $v\in S$, $\pi^{-1}(v)\in S$. Moreover, $\forall v'\in S'$, $\exists v\in S, \pi^{-1}(v)=v'$.

Consider an edge $(u, v)$ in $\gH_{S}$. According to Definition~\ref{def:hassediag}, $u\neq v$,$u\le_{S} v$, and there exists no node $w\in S,w\notin {u, v}$ that $u\le_{S} w$ and $w\le_{S} v$. As $\pi(S')=S$, $\pi^{-1}(u)\neq \pi^{-1}(v)$,$\pi^{-1}(u)\le_{S'} \pi^{-1}(v)$, and there exists no node $\pi^{-1}(w)\in S',\pi^{-1}(w)\notin {\pi^{-1}(u), \pi^{-1}(v)}$ that $\pi^{-1}(u)\le_{S'} \pi^{-1}(w)$ and $\pi^{-1}(w)\le_{S'} \pi^{-1}(v)$. Therefore, when $S=\pi(S')$, for all edge $(u, v)$ in $\gH_{S}$, edge $(\pi^{-1}(u), \pi^{-1}(v))$ exists in $\gH_{S'}$.

Similarly, as $S'=\pi^{-1}(S)$, for all edge $(\pi^{-1}(u),\pi^{-1}(v))$ in $\gH_{S'}$, edge 
\begin{equation*}
((\pi^{-1})^{-1}(\pi^{-1}(u)), (\pi^{-1})^{-1}(\pi^{-1}(v)))=(u, v),
\end{equation*}
exists in $\gH_{S'}$. So $\gH_{S}=\pi(\gH_{S'})$. 
Equivalently, for all $v\in S'$, $\pi(v)$ is in $S$, and $(\{v\}, \gH_{S'})\simeq (\{\pi(v)\}, \gH_{S})$.

Assume that $u,v$ are not isomorphic in $S$, but $\tL(S, \tA)_{u,u,:}=\tL(S,\tA)_{v,v,:}$. Define permutation $\pi: V\to V$ as follows,
\begin{equation*}
    \pi(i)=\begin{cases}
        v& \text{if } i=u\\
        u& \text{if } i=v\\
        i&\text{otherwise}
    \end{cases}.
\end{equation*}
$\pi(\tL(S, \tA))=\tL(S, \tA)\Rightarrow \pi(S)=S\Rightarrow (v, \gH_{S})\simeq (u,\gH_{S})$. Equivalently, non-isomorphic nodes in the same hasse diagram should have different labels.

\subsection{Proof of Theorem~\ref{thm:hypergiso2giso}}
The main gap between hypergraph isomorphism and corresponding graph isomorphism is that hypergraph permutation is composed of two permutation transforms node and edge order independently, while corresponding graph isomorphism is only related to one node permutation, so we first define ways to combine and split permutations.

Sorting of corresponding graph: Let $I_V(IG_H)=\{i|(IG_H)_{i,i,d+1}=1\}$ denote nodes in $G(H)$ corresponding to nodes in $H$. Let $I_E(IG_H)=\{i|(IG_H)_{i,i,d+1}=0\}$ denote the nodes representing hypergraph edges. We define a permutation $\pi^{I_V, I_E}\in \Pi_{n+m}$, $\pi^{I_V, I_E}$, $\pi^{I_V,I_E}(I_V)=[n], \pi^{I_V,I_E}(I_E)=\{n+1,n+2,...,n+m\}$. 

Concatenation of permutation: Let $\pi_1\in\Pi_{n}, \pi_2\in \Pi_m$. Their concatenation $\pi_1|\!|\pi_2\in \Pi_{m+n}$
\begin{equation*}
    \pi_1|\!|\pi_2(i)=
    \begin{cases}
        \pi_1(i)& i\le n\\
        n+\pi_2(i-n)&\text{otherwise}
    \end{cases}
\end{equation*}

When $S_1, S_2$ have different sizes, or $H_1$, $H_2$ have different number of nodes or hyperedges, two poset are non-isomorphic. So we only discuss the case that the poset and hypergraph sizes are the same. Let $n, m$ denote the number of nodes and hyperedges in the hypergraph. Then the corresponding graph has $n+m$ nodes.

We first prove $\Rightarrow$: When $(S, H)\sim (S', H')$, according to Definition~\ref{def:hypergiso}, there exists $\pi_1\in \Pi_n, \pi_2\in \Pi_m, (\pi_1, \pi_2)(H)=H', \pi_1(S)=S'$. Then, $(\pi_1|\!|\pi_2)(IG_{H})=IG_{H'}$ and $(\pi_1|\!|\pi_2)(S)=S'$.

Then we prove $\Leftarrow$: When $(S, IG_H)\simeq (S', IG_{H'})$. We can first sort two incidence graph. Let $\pi=\pi^{I_V(IG_H), I_E(IG_H)}$ and $\pi'=\pi^{I_V(IG_{H'}), I_E(IG_{H'})}$. Then two posets and graphs are still isomorphic.
\begin{equation*}
    (\pi(S), \pi(IG_H))\simeq (\pi'(S'), \pi'(IG_{H'}))
\end{equation*}
Therefore, $\exists \pi_0\in \Pi_{n+m}$, $\pi(S)=\pi_0(\pi'(S')), \pi(IG_H)=\pi_0(\pi'(IG_{H'}))$. Let $\tA, \tA'\in \sR^{(n+m)\times (n+m)\times d+1}$ denote the adjacency tensor of $\pi(IG_H), \pi'(IG_{H'})$ respectively. Therefore,
\begin{equation*}
    \tA=\pi_0(\tA')\Rightarrow \tA_{\pi_0(i), \pi_0(i), d+1}=\tA'_{i,i, d+1}, \forall i\in \{1,2,...,m+n\}.
\end{equation*}
As the nodes in $\tA, \tA'$ are sorted, $\tA_{i,i,d+1}=1, \tA'_{i, i, d+1}=1$ if $i\le n$, and $\tA_{i,i,d+1}=0, \tA'_{i, i, d+1}=0$ if $i>n$. Therefore, $\pi_0$ maps $\{1,2,...,n\}$ to $\{1,2,...,n\}$ and $\{n+1,n+2,...,n+m\}$ to $\{n+1,n+2,...,n+m\}$. Therefore, we can decompose $\pi_0$ into two permutation $\pi_1, \pi_2$.
\begin{equation*}
    \pi_1(i)=\pi_0(i), i\in \{1,2,...,n\}
\end{equation*}
\begin{equation*}
    \pi_2(i)=\pi_0(i+n)-n, i\in \{1,2,...,m\}
\end{equation*}
Then, $S=\pi_1(S')$ and $H=(\pi_1,\pi_2)(H')$.

\subsection{Proof for Section~\ref{sec::labelingtrick_hognn}}

We first define some notations

\textbf{Isomorphism type of node tuple} $k,l$-WL and $k$-WL use the isomorphism type of tuple to initialize colors, which is defined as follows:

Given graphs $G^1=(V^1, \tA^1),G^2=(V^2,\tA^2)$ and $k$-tuples $S^1,S^2$ in $G^1, G^2$ respectively. $S^1, S^2$ have the same isomorphism type iff
\begin{enumerate}
    \item $\forall i_1,i_2\in [k]$,  $S^1_{i_1}=S^1_{i_2}\leftrightarrow\bold S^2_{i_1}=S^2_{i_2}$.
    \item $\forall i,j\in [k], \tA^1_{S^1_iS^1_j}=\tA^2_{S^2_iS^2_j}$.
\end{enumerate} 

\subsubsection{Expressivity comparison}

Given two function $f, g$, $f$ can be expressed by $g$ means that there exists a function $\phi$ that $\phi\circ g=f$, which is equivalent to given arbitrary input $H, G$, $f(H)=f(G)\Rightarrow g(H)=g(G)$. We use $f\to g$ to denote that $f$ can be expressed with $g$. If both $f\to g$ and $g\to f$, there exists a bijective mapping between the output of $f$ to the output of $g$, denoted as $f\leftrightarrow g$.

Here are some basic rule.
\begin{itemize}
\item $g\to h\Rightarrow f\circ g\to f\circ h$.
\item $g\to h, f\to s\Rightarrow f\circ g\to s\circ h$.
\item $f$ is bijective, $f\circ g\to g$
\end{itemize}

\subsubsection{Proof of Proposition~\ref{prop::kwl_lpool}}

The graph color of $k$-WL with $l$-pooling is 
\begin{equation*}
    c_k^{(l)}(G)\text{Hash}(\{\!\{\text{Hash}(\{\!\{{c_k(S\Vert S', G)|S'\in V(G)^{k-l}}\}\!\})|S\in V(G)^l\}\!\})
\end{equation*}

The graph color of $k$-WL with is 
\begin{equation*}
    c_k(G)=\text{Hash}(\{\!\{{c_k(S, G)|S\in V(G)^k}\}\!\}).
\end{equation*}

\begin{equation*}
    c_k^{(l)}(G) \to \{\!\{{c_k(S, G)|S\in V(G)^k}\}\!\}\to  c_k (G)
\end{equation*}

Moreover, as
\begin{align*}
    c_k(S\Vert S', G)&\to \{\!\{ c_k(S\Vert \phi_0(S', v), G)|v\in V(G)
    \}\!\}\\
    &\to \{\!\{ c_k(S\Vert \phi_2(\phi_1(S', v_1), v_2), G)|v_1, v_2\in V(G)
    \}\!\}\\
    &\to ... \to \{\!\{ c_k(S\Vert S')|S'\in V(G)^{k-l}
    \}\!\}
\end{align*}
Therefore,
\begin{align*}
    c_k(G)&\to \{\!\{{c_k(S\Vert S'', G)|S\in V(G)^k, S''\in V(G)^{k-l}}\}\!\} \\
    &\to \{\!\{\{\!\{c_k(S\Vert S', G)|S'\in V(G)^{k-l}\}\!\}S\in V(G)^k, S''\in V(G)^{k-l}\}\!\}\}\!\} 
    \\
    &\to  \{\!\{c_k^{(l)}(G)| S''\in V(G)^{k-l}\}\!\} 
    \to c_k^{(l)}(G)
\end{align*}

\section{Experimental settings}

\textbf{Computing infrastructure.} We leverage Pytorch Geometric and Pytorch for model development. All our models run on an Nvidia 3090 GPU on a Linux server.

\textbf{Hyperparameters} We use Adam optimizer and constant learning rate for all our models. Main hyperparameters for our models are listed in Table~\ref{tab:hyper}. More detailed configuration of each experiments is provided in our code. 

\begin{table}[t]
    \centering
    \begin{tabular}{lccccccc }
    \toprule
      & Data    & BaseGNN & \#layer & hiddim & bs & lr & \#hop \\
    \midrule
      Table~\ref{tab:unlink} & PB, Ecoli   & GIN & 3 & 32 & 32 & 1e-4 & 2 \\
      & Others   & GIN & 3 & 32 & 32 & 1e-4 & 1 \\
      Table~\ref{tab:unlink2} & collab & GIN & 3 & 256 & 32 & 1e-4 & 1\\
      & ddi & GIN & 3 & 96 & 32 & 1e-4 & 1\\
      &citation2 & GIN & 3 & 32 & 32 & 1e-4 & 1\\
      & ppa & GIN & 3 & 32 & 32 & 1e-4 & 1\\
      Table~\ref{tab:directed} & All & GIN & 3 & 32 & 48 & 3e-3&-1\\ 
      Table~\ref{tab:hyper} &  NDC-s, Email-Eu& max & 4 & 64 & 96 & 5e-3 & -1\\
      & Others & max & 4 & 64 & 96 & 4e-3 & -1\\
      Table~\ref{tab:subg} & All & GIN & 1 &64 &64 & 1e-3 & -1\\ 
      \bottomrule
      \end{tabular}
      \caption{Hyperparameters for our models. BaseGNN: GNN used to encoding graph and labels, max means using max aggregator. \#layer: the number of GNN layers. hiddim: hidden dimension. bs: batch size, lr: learning rate. \#hop: the number of hops for sampling subgraph, -1 means using whole graph. }\label{tab:hyper}
\end{table}

\textbf{Model Implementation.} 
For undirected link prediction tasks, our implementation is based on the code of SEAL~\citep{SEAL}, which segregates an ego subgraph from the whole graph for each link. For other tasks, our model runs on the whole graph. We use optuna to perform random search. Hyperparameters were selected to optimize scores on the validation sets. 

\section{More Details about the Datasets}\label{detaileddatasets}

\subsection{Undirected Link Prediction}

We use eight real-world datasets from SEAL~\citep{SEAL}: USAir is a network of US Air lines. NS is a collaboration network of researchers. PB is a network of US political blogs. Power is an electrical grid of western US. Router is a router-level Internet. Ecoli is a metabolic network in E.coli. Cele is a neural network of C.elegans. Yeast is a protein-protein interaction network in yeast.

We also use OGB datasets~\citep{hu2020open}: \texttt{ogbl-ppa}, \texttt{ogbl-collab},  \texttt{ogbl-ddi}, and \texttt{ogbl-citation2}. Among them, \texttt{ogbl-ppa} is a protein-protein association graph where the task is to predict biologically meaningful associations between proteins. \texttt{ogbl-collab} is an author collaboration graph, where the task is to predict future collaborations. \texttt{ogbl-ddi} is a drug-drug interaction network, where each edge represents an interaction between drugs which indicates the joint effect of taking the two drugs together is considerably different from their independent effects. \texttt{ogbl-citation2} is a paper citation network, where the task is to predict missing citations. We present the statistics of these datasets in Table~\ref{data:unlink}. More information about these datasets can be found in \citep{hu2020open}.

\begin{table*}[h]
\caption{Statistics and evaluation metrics of undirected link prediction datasets.}
\label{data:unlink}
\begin{center}
  \resizebox{0.85\textwidth}{!}{
  \begin{tabular}{lccccc}
    \toprule
    \textbf{Dataset}&\textbf{\#Nodes}&\textbf{\#Edges}&\textbf{Avg. node deg.}&\textbf{Split ratio}&\textbf{Metric} \\
    \midrule
    USAir&332 &2,126 &12.81 &0.85/0.05/0.10&auroc\\
NS&1,589 &2,742 &3.45 &0.85/0.05/0.15&auroc\\
PB&1,222 &16,714 &27.36 &0.85/0.05/0.15&auroc\\
Yeast&2,375 &11,693 &9.85 &0.85/0.05/0.15&auroc\\
C.ele&297 &2,148 &14.46 &0.85/0.05/0.15&auroc\\
Power&4,941 &6,594 &2.67 &0.85/0.05/0.15&auroc\\
Router&5,022 &6,258 &2.49 &0.85/0.05/0.15&auroc\\
E.coli&1,805 &14,660 &16.24 &0.85/0.05/0.15&auroc\\
ogbl-ppa&576,289 &30,326,273 &105.25 &fixed&Hits@100\\
ogbl-collab&235,868 &1,285,465 &10.90 &fixed&Hits@50\\
ogbl-ddi&4,267 &1,334,889 &625.68 &fixed&Hits@20\\
ogbl-citation2&2,927,963 &30,561,187 &20.88 &fixed&MRR\\

  \bottomrule
\end{tabular}
}
\end{center}
\end{table*}

\subsection{Directed Link Prediction}
We use the same settings and datasets as ~\citet{PyGsd}. The task is to predict whether a directed link exists in a graph. Texas, Wisconsin, and Cornell consider websites as nodes and links between websites as edges. Cora-ML and CiteSeer are citation networks. Telegram is an influence graph between Telegram channels. Their statistics are shown in Table~\ref{data:dirlink}.

\begin{table*}[h]
\caption{Statistics and evaluation metrics of directed link prediction datasets.}
\label{data:dirlink}
\begin{center}
  \begin{tabular}{lccccc}
\toprule\textbf{Dataset}&\textbf{\#Nodes}&\textbf{\#Edges}&\textbf{Avg. node deg.}&\textbf{Split ratio}&\textbf{Metric} \\
    \midrule
    wisconsin&251 &515 &4.10 &0.80/0.05/0.15&accuracy\\
    cornell&183 &298 &3.26 &0.80/0.05/0.15&accuracy\\
texas&183 &325 &3.55 &0.80/0.05/0.15&accuracy\\
cora\_ml&2,995 &8,416 &5.62 &0.80/0.05/0.15&accuracy\\
telegram&245 &8,912 &72.75 &0.80/0.05/0.15&accuracy\\
citeseer&3,312 &4,715 &2.85 &0.80/0.05/0.15&accuracy\\
  \bottomrule
\end{tabular}
\end{center}
\end{table*}

\subsection{Hyperedge Prediction Datasets}
We use the datasets and baselines in \citep{familyset}. NDC-c (NDC-classes) and NDC-s (NDC-substances) are both drug networks. NDC-c takes each class label as a node and the set of labels applied to a drug as a hyperedge. NDC-s takes substances as nodes and the set of substances contained in a drug as a hyperedge. Tags-m (tags-math-sx) and tags-a (tags-ask-ubuntu) are from online Stack Exchange forums, where nodes are tags and hyperedges are sets of tags for the same questions. Email-En (email-Enron) and email-Eu are two email networks where each node is a email address and email hyperedge is the set of all addresses on an email. Congress (congress-bills) takes Congress members as nodes, and each hyperedge corresponds to the set of members in a committe or cosponsoring a bill. Their statistics are shown in Table~\ref{data:hyper}.

\begin{table*}[h]
\caption{Statistics and evaluation metrics of directed link prediction datasets.}
\label{data:hyper}
\begin{center}
  \begin{tabular}{lcccc}
\toprule\textbf{Dataset}&\textbf{\#Nodes}&\textbf{\#Hyperdges}&\textbf{Split ratio}&\textbf{Metric} \\
    \midrule
    NDC-c&6,402 &1,048 &5-fold&f1-score\\
NDC-s&49,886 &6,265 &5-fold&f1-score\\
tags-m&497,129 &145,054 &5-fold&f1-score\\
tags-a&591,904 &169,260 &5-fold&f1-score\\
email-En&4,495 &1,458 &5-fold&f1-score\\
email-EU&85,109 &24,400 &5-fold&f1-score\\
congress&732,300 &83,106 &5-fold&f1-score\\
  \bottomrule
\end{tabular}
\end{center}
\end{table*}

\subsection{Subgraph Prediction Tasks}

Following~\citep{GLASS}, we use three synthetic datasets: density, cut ratio, coreness. The task is to predict the corresponding properties of randomly selected subgraphs in random graphs. Their statistics are shown in Table~\ref{data:subg}.

\begin{table*}[h]
\caption{Statistics and evaluation metrics of directed link prediction datasets.}
\label{data:subg}
\begin{center}
  \begin{tabular}{lccccc}
\toprule\textbf{Dataset}&\textbf{\#Nodes}&\textbf{\#Edges}&\textbf{\#Subgraphs}&\textbf{Split ratio}&\textbf{Metric} \\
    \midrule
density&5,000 &29,521 &250 &0.50/0.25/0.25&f1-score\\
cut-ratio&5,000 &83,969 &250 &0.50/0.25/0.25&f1-score\\
coreness&5,000 &118,785 &221 &0.50/0.25/0.25&f1-score\\
  \bottomrule
\end{tabular}
\end{center}
\end{table*}

\section{Time and GPU Memory in Link Prediction Task}\label{app::time}

To illustrate the scalability of GNNs, we measure the time and GPU memory consumption on ppa dataset. The process we measure including all precomputation and prediction a number of edges in one batch. The results are shown in Figure~\ref{fig:time}. For GNNs with labeling tricks (ZO-S, ZO-OS, ZO, SEAL) and GNN without labeling trick for ablation (No), they all have nearly the same time and memory consumption, as the only difference is integer label computation and one embedding layer for encoding labels. They all sample subgraphs from the whole graph and do not need to precompute embeddings for all nodes in the graph, so when the number of edges is small, the time and memory approaches $0$. In contrast, GAE precomputes all nodes' embeddings, leading to large time and GPU consumption even for few edges. It has lower time and GPU consumption after the precomputation. For large real-world graphs, putting whole graphs into memory is impossible and thus sampling subgraphs is a must (even for GNNs without labeling trick), so labeling trick will not introduce a high extra cost.

\begin{figure}
    \centering
    \includegraphics[width=1.0\linewidth]{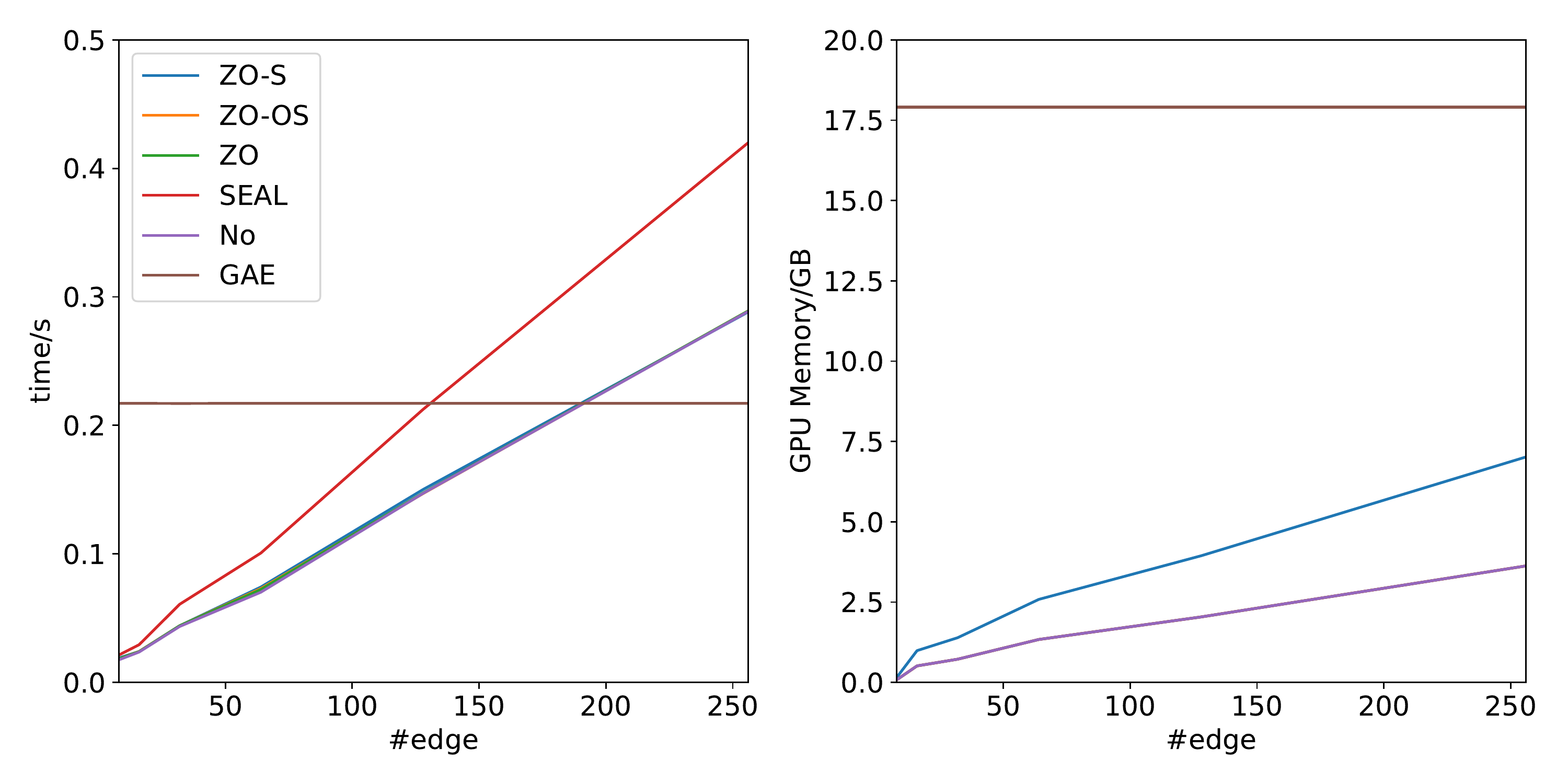}
    \caption{Time and GPU memory consumption for link prediction task on ppa dataset.}
    \label{fig:time}
\end{figure}

\bibliography{subsample}

\end{document}